\documentclass{article}
\pdfoutput=1
\usepackage[preprint]{neurips_2024_custom}

\usepackage[utf8]{inputenc} % allow utf-8 input
\usepackage[T1]{fontenc}    % use 8-bit T1 fonts
\usepackage{hyperref}       % hyperlinks
\usepackage{url}            % simple URL typesetting
\usepackage{booktabs}       % professional-quality tables
\usepackage{amsfonts}       % blackboard math symbols
\usepackage{nicefrac}       % compact symbols for 1/2, etc.
\usepackage{microtype}      % microtypography
\usepackage{xcolor}         % colors
\usepackage{xspace}
\usepackage{amsmath}
\usepackage{amsthm}
\usepackage{caption}
\usepackage{subcaption}

\usepackage{cleveref}

\usepackage{bm}

\theoremstyle{plain}

\newtheorem{lemma}{Lemma}
\newtheorem{theorem}{Theorem}
\usepackage[ruled]{algorithm2e} 
\usepackage{algpascal}

\usepackage{enumitem}

\usepackage{wrapfig}
\usepackage{comment}
\usepackage{selectp}
\usepackage{selectp}

\newcommand{\SVI}{SVI\xspace}

\newcommand{\SPI}{SPI\xspace}

\newcommand{\muout}{\mu_{\text{out}}}

\newcommand{\piout}{{\pi_{\text{out}}}}

\newcommand{\PiHD}{\Pi_{\text{HD}}}

\newcommand{\tpi}{\wt{\pi}}
\newcommand{\tmu}{\wt{\mu}}

\newcommand{\bmu}{\overline{\mu}}

\newcommand{\HH}[2]{\mathcal{H}\pa{#1\middle\|#2}}

\newcommand{\CC}{\mathcal{C}}

\newcommand{\diag}{\mathrm{diag}}
\newcommand{\dd}{\mathrm{d}}

\newcommand{\VX}{V_\mathcal{X}}
\newcommand{\VY}{V_\mathcal{Y}}
\newcommand{\MX}{M_\mathcal{X}}
\newcommand{\MY}{M_\mathcal{Y}}
\newcommand{\MXY}{M_{\X\Y}}

\newcommand{\Pibc}{\Pi_{\text{bc}}}

\newcommand{\X}{\mathcal{X}}

\newcommand{\PX}{\mathcal{P}_X}
\newcommand{\IX}{\mathcal{I}_\mathcal{X}}

\newcommand{\Y}{\mathcal{Y}}

\newcommand{\M}{\mathcal{M}}
\newcommand{\A}{\mathcal{A}}
\newcommand{\LL}{\mathcal{L}}

\newcommand{\BB}{\mathcal{B}}
\newcommand{\TT}{\mathcal{T}}
\newcommand{\TTpi}{\mathcal{T}^{\pi}}
\newcommand{\TTpik}{\mathcal{T}^{\pi_k}}

\newcommand{\BX}{\mathcal{B}_\mathcal{X}}
\newcommand{\BY}{\mathcal{B}_\mathcal{Y}}
\newcommand{\SX}{\mathcal{S}_\mathcal{X}}
\newcommand{\SY}{\mathcal{S}_\mathcal{Y}}
\newcommand{\PPX}{P_\mathcal{X}}
\newcommand{\PPY}{P_\mathcal{Y}}
\newcommand{\nuzerox}{\nu_{0,\X}}
\newcommand{\nuzeroy}{\nu_{0,\Y}}

\newcommand{\F}{\mathcal{F}}

\newcommand{\real}{\mathbb{R}}

\newcommand{\Sw}{\mathcal{S}}

\newcommand{\DD}[2]{\mathcal{D}\pa{#1\middle\|#2}}

\newcommand{\DDKL}[2]{\mathcal{D}_{\mathrm{KL}}\pa{#1\middle\|#2}}

\newcommand{\II}[1]{\mathbb{I}_{\left\{#1\right\}}}
\newcommand{\PP}[1]{\mathbb{P}\left[#1\right]}

\newcommand{\EE}[1]{\mathbb{E}\left[#1\right]}

\newcommand{\EEpi}[1]{\mathbb{E}_{\pi}\left[#1\right]}
\newcommand{\PPpi}[1]{\mathbb{P}_{\pi}\left[#1\right]}

\newcommand{\PPcc}[2]{\mathbb{P}\left[\left.#1\right|#2\right]}

\newcommand{\EEcs}[3]{\mathbb{E}_{#3}\left[\left.#1\right|#2\right]}

\def\argmin{\mathop{\mbox{ arg\,min}}}

\newcommand{\ra}{\rightarrow}

\newcommand{\iprod}[2]{\left\langle#1,#2\right\rangle}

\newcommand{\norm}[1]{\left\|#1\right\|}

\newcommand{\onenorm}[1]{\norm{#1}_1}

\newcommand{\infnorm}[1]{\norm{#1}_\infty}

\newcommand{\ev}[1]{\left\{#1\right\}}
\newcommand{\abs}[1]{\left|#1\right|}

\newcommand{\pa}[1]{\left(#1\right)}
\newcommand{\bpa}[1]{\bigl(#1\bigr)}

\newcommand{\wt}{\widetilde}

\newcommand{\bX}{\overline{X}}
\newcommand{\bY}{\overline{Y}}
\newcommand{\by}{\overline{y}}
\newcommand{\bx}{\overline{x}}

\newcommand{\transpose}{^\mathsf{\scriptscriptstyle T}}

\usepackage{todonotes}
\definecolor{PalePurp}{rgb}{0.66,0.57,0.66}

\usepackage[normalem]{ulem}
\usepackage{enumitem}

\newtheorem{prop}{Proposition}

\title{Bisimulation Metrics are Optimal Transport Distances, \\ and Can be Computed Efficiently}

\author{Sergio Calo\,\,\,\, Anders Jonsson \,\,\,\, Gergely Neu \,\,\,\, Ludovic Schwartz 
\,\,\,\, Javier Segovia-Aguas
 \\
 Universitat Pompeu Fabra, Barcelona, Spain\\
 {\small\texttt{\{sergio.calo,anders.jonsson,gergely.neu,ludovic.schwartz,javier.segovia\}@upf.edu}}
 }
 
\begin{document}

\maketitle

\begin{abstract}
We propose a new framework for formulating optimal transport
distances between Markov chains. Previously known formulations studied
couplings between the entire joint distribution induced by the chains,
and derived solutions via a reduction to dynamic programming (DP) in
an appropriately defined Markov decision process. This formulation
has, however, not led to particularly efficient algorithms so far,
since computing the associated DP operators requires fully solving a
static optimal transport problem, and these operators need to be
applied numerous times during the overall optimization process. In
this work, we develop an alternative perspective by considering
couplings between a ``flattened'' version of the joint distributions
that we call discounted occupancy couplings, and show that calculating
optimal transport distances in the full space of joint distributions
can be equivalently formulated as solving a linear program (LP) in
this reduced space. This LP formulation allows us to port
several algorithmic ideas from other areas of optimal transport
theory. In particular, our formulation makes it possible to introduce
an appropriate notion of entropy regularization into the optimization
problem, which in turn enables us to directly calculate optimal
transport distances via a Sinkhorn-like method we call Sinkhorn Value
Iteration (SVI). We show both theoretically and empirically that this
method converges quickly to an optimal coupling, essentially at the
same computational cost of running vanilla Sinkhorn in each pair of
states. Along the way, we point out that our optimal transport
distance exactly matches the common notion of bisimulation metrics
between Markov chains, and thus our results also apply to computing
such metrics, and in fact our algorithm turns out to be significantly
more efficient than the best known methods developed so far for this
purpose.
\end{abstract}

\section{Introduction}\label{sec:intro}

Measuring distances between structured objects and sequences is an important problem in a variety of areas of 
science. The more structured the objects become, the harder it gets to define appropriate notions of distances, as 
good notions of proximity need to take into account the possibly complex relationships between the constituent 
parts of each object. The possibility that the objects in question may be random further complicates the picture, and 
in such cases it becomes more natural to measure distances between the underlying joint probability distributions. 
Within the specific context of comparing stochastic processes, two natural notions of distance have emerged over the 
past decades: the notion of \emph{probabilistic bisimulation metrics} that takes its root in modal logic and 
theoretical computer science \citep{sangiorgi2009origins,Aba13}, and the notion of \emph{optimal-transport distances} 
that originates from probability theory \citep{villani2009optimal,COTFNT}.
In this paper, we show that bisimulation metrics are in fact optimal-transport distances, and we make use of this 
observation to derive efficient algorithms for computing distances between stochastic processes.

The two distance notions have found strikingly different applications. Bisimulation emerged within the area of 
theoretical computer science as one of the most important important concepts in concurrency theory and formal 
verification of computer systems \citep{Park81,Milner89}, and has been extended to probabilistic transition systems by 
\citet{Larsen1989BisimulationTP}. Within machine learning, bisimulation metrics have become especially popular in the 
context of reinforcement learning (RL) due to the work of \citet{FPP04}, and have become one of the few standard tools of 
representation learning \citep{Jia18,jiang2024}. In particular, the work of \citet{FPP04} advocates for using 
bisimilarity 
as a basis for state aggregation, measuring similarities of states in terms of similarities of two chains $\MX$ and 
$\MY$ that only differ in their initial state. While this approach has inspired numerous follow-up works 
\citep{gelada2019,castro2020scalable,agarwal2021,zhang2021,hansen2022bisimulation,castro2022kernel}, ultimately this 
line of work has failed to discover efficient algorithms for computing bisimulation metrics and has largely resorted to 
heuristics for computing similarity metrics.\looseness=-1

On the other hand, optimal transport (OT) has found numerous applications in areas as diverse as economics 
\citep{galichon2016}, signal processing \citep{kolouri2017}, or genomics \citep{Sch+19}. Within machine learning, it has 
been used for the similarly diverse areas of domain adaptation \citep{CFTR16}, 
generative modeling \citep{ACB17,SSKKEP20,SDCD24}, representation learning \citep{CFD18}, and, perhaps most relevant to 
our work, as a way of measuring distances between graphs \citep{TCTF19,CLMWW22,CJ22}. The recent works of 
\cite{YOMN21,BWW24} propose to define graph distances via studying the behavior of random walks defined on the graph, 
thus reducing the problem of comparing graphs to comparing stochastic processes---exactly the subject of the present 
paper. Other applications of OT between stochastic processes include generative modeling for sequential data 
\citep{XWMA20}, pricing and hedging in mathematical finance \citep{BBLZ17}, and analyzing multistage 
stochastic optimization problems \citep{Pfl10,BW22}. It appears however that the literature on optimal transport for 
stochastic processes has apparently not yet discovered connections with bisimulation metrics and the rich intellectual 
history behind it. Also, applications of optimal transport for representation learning within the context 
of reinforcement learning appear to be nonexistent.

In this paper we observe that, despite their apparent differences, bisimulation metrics and optimal transport 
distances are one and the same. Furthermore, we provide a new perspective on both OT distances and 
bisimulation metrics by formulating the distance metric as the solution of a linear program (LP) in the space of 
``occupancy couplings'', a finite-dimensional projection of the infinite-dimensional process laws. Building on tools 
from computational optimal transport \citep{COTFNT} and entropy-regularized Markov decision processes 
\citep{NJG17,GSP19}, we design an algorithm that effectively combines Sinkhorn's algorithm 
\citep{sinkhorn1967concerning,cuturi2013sinkhorn} with an entropy-regularized version of the classic Value Iteration 
algorithm \citep{Bel57,NJG17}. Building on recent work on computational optimal transport 
\citep{altschuler2017near,BB23}, we provide theoretical guarantees for the resulting algorithm (called 
Sinkhorn Value Iteration) and perform numerical studies that demonstrate its effectiveness for computing distances 
between Markov chains.

\paragraph{Notations.} For a finite set $\Sw$, we use $\Delta_\Sw$ to denote the set of all probability distributions 
over $\Sw$. We will denote infinite sequences by $\overline{x} = \pa{x_0,x_1,\dots}$ and the corresponding subsequences 
as $\overline{x}_n = \pa{x_0,x_1,\dots,x_n}$. For two sets $\X$ and $\Y$, we will often write $\X\Y$ to abbreviate the 
direct-product notation $\X\times\Y$, and for two indices $x$ and $y$ and a function $f:\X\Y\ra\mathcal{Z}$, we will 
often write $f(xy)$ instead of $f(x,y)$ to save space.
Also, we will denote scalar products by $\langle \cdot,\cdot \rangle$ and use $\norm{\cdot}_p$ to denote the 
$\ell_p$-norm.

\section{Preliminaries}\label{sec:prelim}
We study the problem of measuring distances between pairs of finite Markov chains. Specifically, we consider 
two stationary Markov processes $\MX=(\mathcal{X}, \PPX, \nuzerox)$ and $\MY=(\mathcal{Y}, \PPY, \nuzeroy)$, where 
\begin{itemize}[topsep=-1mm,itemsep=-.5mm,partopsep=1mm,parsep=1mm,leftmargin=5mm]
 \item $\X$ and $\Y$ are the finite state spaces with cardinalities $m = \abs{\X}$ and $n = \abs{Y}$,
 \item $\PPX:\X \rightarrow \Delta_{\X}$  and $\PPY:\Y \rightarrow \Delta_{\Y}$ are the transition kernels that 
determine the evolution of the states as $\PPX(x'|x) = \PPcc{X_{t+1} \!=\! x'}{X_t \!=\! x}$ 
and $\PPY(y'|y) = \PPcc{Y_{t+1} \!=\! y'}{Y_t \!=\! y}$ for all $t$, and
 \item $\nuzerox$ and $\nuzeroy$ are the initial-state distributions with $X_0\sim \nuzerox$ and 
$Y_0\sim \nuzeroy$.
\end{itemize}
Without significant loss of generality, we will suppose that the initial states are fixed almost surely as $X_0 = x_0$ 
and $Y_0 = y_0$, and refer to their corresponding joint distribution as $\nu_0 = \delta_{x_0,y_0}$. These objects 
together define a sequence of joint distributions $\PP{(X_0,X_1,\dots,X_n) = (x_0,x_1,\dots,x_n)}$ and 
$\PP{(Y_0,Y_1,\dots,Y_n) = (y_0,y_1,\dots,y_n)}$ for each $n$, which together define respective the laws of the 
infinite sequences $\bX = (X_1,X_2,\dots)$ and $\bY = (Y_1,Y_2,\dots)$ via Kolmogorov's extension theorem. With a 
slight abuse of notation, we will use $\MX$ and $\MY$ to denote the corresponding measures that satisfy 
$\MX(\overline{x}_n) = \PP{\bX_n = \bx_n}$ and $\MY(\overline{y}_n) = \PP{\bY_n = \by_n}$ for any $\bx\in\X^\infty$, 
$\by\in\Y^\infty$ and $n$. The corresponding conditional distributions are denoted as $\MX(x_n|\bx_{n-1}) = \PP{X_n = 
x_n | \bX_{n-1}=\bx_{n-1}}$ and $\MY(y_n|\by_{n-1}) = \PP{Y_n = y_n | \bY_{n-1}=\by_{n-1}}$.

\subsection{Optimal transport between Markov chains}\label{sec:OT_markov}
Our main object of interest in this work is a notion of optimal transport distance between infinite-horizon Markov 
chains. Several previous works have studied such distances (which are discussed in detail in 
Appendix~\ref{app:related_work}), and our precise definition we give below is closest to \citet{moulos2021bicausal}, 
\citet{OMN21} and \citet{BWW24}. 
We consider Markov chains on state spaces where a ``ground metric'' (or ``ground cost'') $c:\X\times\Y \rightarrow 
\real^+$ is available to measure distances between any two individual states $x\in\X$ and $y\in\Y$, with the 
distance denoted as $c(x,y)$.
For any two sequences $\overline{x} = (x_1,x_2,\dots)$ and $\overline{y} = (y_1,y_2,\dots)$, we define the discounted 
total cost 
\begin{equation}
	c_\gamma(\overline{x}, \overline{y}) = \sum_{t=0}^{\infty} \gamma^t c(x_t, y_t),
\end{equation}
where $\gamma \in (0,1)$ is the \emph{discount factor} that expresses the preference that two sequences be considered 
further apart if they exhibit differences at earlier times in terms of the ground cost $c$. 
Following the optimal-transport literature, we will consider distances between the stochastic processes $\MX$ and $\MY$ 
via the notion of \emph{couplings}. To this end, we define a coupling of $\MX$ and $\MY$ as a stochastic process 
evolving on the joint space $\mathcal{X}\mathcal{Y}$, with its law defined for all $n$ as 
$\MXY(\bx_n\by_n) = \mathbb{P}\big[\bX_n=\bx_n,\bY_n=\by_n\big]$, required to satisfy 
$\sum_{\by_n\in\Y^n} \MXY(\bx_n\by_n) = \MX(\bx_n)$ and $\sum_{\bx_n\in\X^n} \MXY(\bx_n\by_n) = 
\MY(\by_n)$. We will define the set of all such couplings as $\Pi$.

The notion of couplings defined above does not respect the temporal structure of the Markov chains $\MX$ and $\MY$ 
appropriately: while by definition the distribution of state $X_n$ may only be causally influenced by past states $X_k$ 
with $k<n$, the general notion of coupling above allows the state $X_n$ to be influenced by future states $Y_k$ 
with $k\ge n$ as well. To rule out this possibility (and following past works mentioned in the introduction), we will 
introduce the notion of \emph{bicausal couplings}. A coupling $\MXY$ is called bicausal if and only if it satisfies
\[
 \sum_{y_n} \MXY(x_ny_n|\bx_{n-1}\by_{n-1}) = \MX(x_n|\bx_{n-1}) \;\; \text{and} \;\; \sum_{x_n} \MXY(x_ny_n|\bx_{n-1}\by_{n-1}) = \MY(y_n|\by_{n-1})
\]
for all sequences $\bx,\by\in\X^\infty\times\Y^\infty$ and all $n$. Denoting the set of all bicausal couplings by 
$\Pibc$, we define our optimal transport distance as
\begin{equation}
\label{eq:objective}
    \mathbb{W}_\gamma(\MX,\MY;c,x_0,y_0) = \inf_{\MXY \in \Pibc} \int c_\gamma(\bX, \bY) \,\dd \MXY(\bX,\bY),
\end{equation}
where we emphasize the dependence of the distance on $x_0,y_0$ explicitly with our notation.

By noticing that the optimization problem outlined above can be reformulated as a Markov decision process (MDP), 
\citet{moulos2021bicausal} has shown that the infimum in \eqref{eq:objective} is achieved within the family of Markovian couplings that satisfy
\[
 \sum_{y_n} \MXY(x_ny_n|\bx_{n-1}\by_{n-1}) = \PPX(x_n|x_{n-1}) \;\; \text{and} \;\; \sum_{x_n} \MXY(x_ny_n|\bx_{n-1}\by_{n-1}) = \PPY(y_n|y_{n-1})
\]
for all sequences of state pairs and all values of $n$. Furthermore, it can be seen that Markovian couplings are fully 
specified in terms of \emph{transition couplings} of the form $\pi:\X\Y\ra\Delta_{\X\Y}$, with 
$\pi(x'y'|xy)$ standing for $\PPcc{\pa{X_{t+1},Y_{t+1}} = (x',y')}{\pa{X_t,Y_t}=(x_t,y_t)}$ under the law induced by 
the coupling. We say that a transition coupling is \emph{valid} if it satisfies the marginal constraints $\sum_{y'} 
\pi(x'y'|xy) = 
\PPX(x'|x)$ and $\sum_{x'} \pi(x'y'|xy) = \PPY(y'|y)$. Defining the set of such valid transition couplings in state 
pair $xy$ by $\Pi_{xy} = \big\{p \in \Delta_{\X\Y}: \sum_{y'} p(x'y') = \PPX(x'|x),\,\,\sum_{x'} p(x'y') = \PPY(y'|y)\big\}$, 
\citet{moulos2021bicausal} introduces an MDP $\M$ with an infinite action set corresponding to picking the joint next-state couplings in $\Pi_{xy}$.
An optimal transition coupling can be found by solving the following Bellman optimality equations of the MDP $\M$: 
\begin{equation}\label{eq:extended_bellman}
 V^*(xy) = c(xy) + \gamma \inf_{p\in\Pi_{xy}} \sum_{x'y'} p(x'y') V^*(x'y').
\end{equation}
The infimum on the right-hand side is achieved by an optimal transition coupling $\pi^*(\cdot|xy) = 
\argmin_{p\in\Pi_{xy}} \sum_{x'y'} p(x'y') V^*(x'y')$. The solution $V^*$ is unique and can be shown to satisfy 
$V^*(xy) = \mathbb{W}_\gamma(\MX,\MY;c,x,y)$ for all $xy\in\X\Y$.
For completeness, we include the precise definition of the MDP $\M$ and the proofs of these 
results in Appendix~\ref{app:mdp}.

\subsection{Bisimulation metrics}
The notion of \emph{bisimulation metrics} has been introduced by \citet{DGJP99,DGJP04} and \citet{vBW01}, with the 
purpose of defining distances between Markov chains, using a methodology rooted in modal logic that at first may appear 
entirely different from the optimal-transport framework described above. We only give a very high-level overview of the 
classic logic-based characterization here (as the fine details are irrelevant to the final conclusion that this section 
is headed to), and refer the reader to the additional discussion in Appendix~\ref{app:related_work} for further reading.
 \citet{DGJP99,DGJP04} 
considered \emph{labeled} Markov chains where a labeling function $r:\X\cup\Y\ra \real$ assigns labels to each state, 
and defined bisimulation metrics via
\begin{equation}\label{eq:bisimulation_metric}
 d_\gamma(\MX,\MY;r,x_0,y_0) = \sup_{f\in\F_\gamma} \left|f_{\MX}(x_0) - f_{\MY}(y_0)\right|,
\end{equation}
where $\F_\gamma$ is a family of functional expressions generated by a certain grammar, and $f_{\MX}:\X\ra\real$ and 
$f_{\MY}:\Y\ra\real$ are the respective ``interpretations'' of each $f\in\F$ on the Markov chains $\MX$ and $\MY$. 
Building on this formulation, \citet{vBW01} have shown that the distance metric can be equivalently characterized by 
the solution of a fixed-point equation, whose expression was subsequently used by \citet{FPP04} to define bisimulation
metrics for Markov \emph{decision} processes where $r$ takes the role of a reward function, and the evolution of 
states may be influenced by actions. The case of having no actions available corresponds to our setting, where their 
definition of a bisimulation metric simplifies to the solution of the fixed-point equation 
\begin{equation}\label{eq:extended_bellman_v2}
 U^*(xy) = (1-\gamma)|r(x) - r(y)| + \gamma \inf_{p\in\Pi_{xy}} \sum_{x'y'} p(x'y') U^*(x'y').
\end{equation}
The solution to this system is unique and satisfies $U^*(xy) = d_\gamma(\MX,\MY;r,x,y)$.
Putting this result side-by-side with Equation~\eqref{eq:extended_bellman}, one can immediately realize that 
\emph{bisimulation metrics and our notion of optimal transport distances coincide when picking the ground cost function 
$c(xy) = (1-\gamma) |r(x) - r(y)|$}. To our knowledge, this remarkable observation has not been publicly made anywhere 
in either the optimal-transport or the bisimulation-metric literature. This connection has several important 
implications, which we have already discussed at some length in Section~\ref{sec:intro}. We relegate further discussion 
of these metrics in the light of this observation to Appendix~\ref{app:related_work}.

\section{Optimal transport between Markov chains as a linear program}\label{sec:LP}
Optimal transport problems can typically be formulated as linear programs (LPs), since couplings can be characterized 
as joint distributions satisfying a set of easily-expressed linear constraints (see, e.g., Chapter~3 in \citealt{COTFNT}). 
Our problem is no exception, 
and in fact the original problem statement of Equation~\ref{eq:objective} can be expressed in this form: 
the constraints defining $\Pibc$ are all linear in $\MXY$. However, $\MXY$ is an infinite-dimensional object and thus 
this formulation is not instructive for developing computationally tractable algorithms. We address this problem in 
this section, where we define an equivalent LP formulation that replaces the infinite-dimensional 
optimization variable with an appropriate low-dimensional projection. Our framework builds on the classic LP formulation 
of optimal control in MDPs first proposed in the 1960's \citep{Man60,Ghe60,dEp63,Den70}, and covered thoroughly in 
several standard textbooks (e.g., Section~6.9 of \citealt{Puterman1994}). 

In order to set things up, we need to start with some important definitions. We say that a transition coupling 
$\pi:\X\Y\ra\Delta_{\X\Y}$ 
generates a trajectory $(X_0,Y_0,X_1,Y_1,\dots)$ if $(X_0,Y_0) \sim \nu_0$ and the subsequent state-pairs are drawn 
independently from the transition coupling as $(X_{t+1},Y_{t+1})\sim \pi(\cdot|X_t,Y_t)$ for all $t$. Then, we define 
the \emph{occupancy coupling} $\mu^\pi$ associated with this process as a distribution over $\X\Y\times\X\Y$, with each of 
its entries $xy,x'y'$ defined as
\[
 \mu^\pi(xy,x'y') = (1-\gamma)\EEpi{\sum_{t=0}^\infty \gamma^t \II{(X_t,Y_t) = (x,y), (X_{t+1},Y_{t+1}) = (x',y')}},
\]
where $\EEpi{\cdot}$ emphasizes that the trajectory of state-pairs has been generated by $\pi$.
In words, $\mu^\pi(xy,x'y')$ is the discounted number of times that the quadruple $(xy,x'y')$ is visited by the 
process. With this definition, it is easy to notice that the objective optimized in Equation~\eqref{eq:objective} can 
be rewritten as a linear function of $\mu^\pi$. Indeed, suppose that $\MXY$ is the law of the process generated by 
$\pi$ as described above, so that we can write
\begin{equation*}
 \begin{split}
  &\int c_\gamma(\bX, \bY) \,\dd \MXY(\bX,\bY) = \EEpi{\sum_{t=0}^\infty \gamma^t c(X_t,Y_t)} = 
  \EEpi{\sum_{t=0}^\infty \gamma^t \sum_{xy} \II{(X_t,Y_t)=(x,y)} c(xy)} 
  \\
  &\qquad = \sum_{xy} \EEpi{\sum_{t=0}^\infty \gamma^t \II{(X_t,Y_t)=(x,y)}} 
c(xy) = \frac{1}{1-\gamma}\sum_{xy,x'y'} \mu^{\pi}(xy,x'y') c(xy) = \frac{\iprod{\mu^\pi}{c}}{{1-\gamma}}. 
 \end{split}
\end{equation*}

We say that an occupancy coupling $\mu$ is \emph{valid} if it is generated by a valid transition coupling. 
It is easy to verify that every valid occupancy coupling $\mu\in\Delta_{\X\Y\times\X\Y}$ satisfies the following 
three constraints:
\begin{align}
\sum_{x'y'} \mu(xy,x'y') &= \gamma \sum_{x'y'} \mu(x'y',xy) + (1-\gamma) \nu_0(xy) \qquad (\forall 
xy\in\X\Y), \label{eq:const_flow}
\\
\sum_{y'} \mu(xy,x'y') &= \sum_{x'',y''} \mu(xy,x''y'') \PPX(x'|x) 
\qquad (\forall x,x'\in\X\times\X), \label{eq:const_X}
\\
\sum_{x'} \mu(xy,x'y') &= \sum_{x'',y''} \mu(xy,x''y'') \PPY(y'|y)
\qquad (\forall 
y,y'\in\Y\times\Y). \label{eq:const_Y}
\end{align}

We refer to the first equality constraint as the \emph{flow constraint}
and the second and third ones as the \emph{transition coherence constraints} for $\MX$ and 
$\MY$, respectively\footnote{When $\gamma = 1$, the first condition is known as \emph{detailed balance} within 
statistical physics. Altogether, the constraints closely resemble what 
are often called the ``Bellman flow constraints'' in today's RL literature.}.
We show in the following lemma that the above conditions uniquely characterize the set of valid occupancy couplings.
\begin{lemma}\label{lem:occupancy_validity}
A distribution $\mu\in\Delta_{\X\Y\times\X\Y}$ is a valid occupancy coupling associated with some transition coupling 
$\pi:\X\Y\ra\Delta_{\X\Y}$ if and only if it satisfies Equations~\eqref{eq:const_flow}--\eqref{eq:const_Y}.
\end{lemma}
One important consequence of this result is that for each transition coupling, one can associate a transition coupling 
$\pi_\mu$ with $ \mu^{\pi_\mu} = \mu $, with entries  satisfying $\sum_{x''y''} 
\mu(xy,x''y'')\pi_\mu(x'y'|xy) = \mu(xy,x'y')$.
A proof is provided in Appendix~\ref{app:occupancy_validity}. Having established in the 
previous section that stationary occupancy couplings are sufficient to achieve the supremum in the definition of the OT 
distance of Equation~\eqref{eq:objective}, this result immediately implies the following.
\begin{theorem}\label{thm:LP}
Let $\BB$, $\SX$ and $\SY$ respectively stand for the sets of all distributions $\mu\in\Delta_{\X\Y\times\X\Y}$ that satisfy 
equations \eqref{eq:const_flow}, \eqref{eq:const_X}, and \eqref{eq:const_Y}. Then, 
\[
 \mathbb{W}_\gamma(\MX,\MY;c,x_0,y_0) = \frac{1}{1-\gamma} \cdot \inf_{\mu\in\BB\cap\SX\cap\SY} \iprod{\mu}{c}.
\]
\end{theorem}
Since the constraints on $\mu$ in the above reformulation are all linear, the optimization problem stated above is 
obviously a linear program. Notably, the resulting LP is \emph{not} the standard dual to the LP associated with the MDP 
formulation introduced in Section~\ref{sec:OT_markov}, which would result in an infinite-dimensional LP with one 
constraint associated with each of the continuously-valued actions. Such an infinite-dimensional reformulation was 
previously considered by \citet{CvBW12} in the context of computing bisimulation metrics, who 
used it as a tool for analysis rather than algorithm design. Notably, our formulation results in a 
finite-dimensional LP with $\abs{\X}^2\abs{\Y}^2$ variables and $\abs{\X}^2 + \abs{\X}\abs{\Y} + \abs{\Y}^2$ 
constraints. Building on the celebrated result of \citet{Ye11} (and similarly to the work of \citealp{CvBW12} mentioned 
above), one can show that our LP can be solved in strongly polynomial time via an appropriate adaptation of the simplex 
method or the classic policy iteration method of \citet{How60}. We propose an alternative methodology in the next 
section.

\section{Sinkhorn Value Iteration}
\label{sec:SVI}
Solving optimal-transport problems via standard LP solvers (like variations of the simplex method or interior-point 
methods) is known to be empirically hard, and thus we seek alternatives to this approach towards optimizing our own LP 
defined in the previous section. In computational optimal transport, a paradigm shift was initiated by 
\citet{cuturi2013sinkhorn} who successfully applied entropic regularization to the classic LP objective of optimal 
transport, and solved the resulting optimization method through an iterative algorithm called the Sinkhorn--Knopp method 
(sometimes simply called Sinkhorn's algorithm, due to \citealp{sinkhorn1967concerning}). This method is based on 
finding a feasible point in a two-constraint problem by alternately satisfying one and the other, and has resulted in 
practical algorithms that were orders of magnitude faster than all previously studied methods. Entropy regularization 
and Sinkhorn's algorithm have thus became the most important cornerstones of computational optimal transport. Drawing 
on the same principles as well as the theory of entropy-regularized Markov decision processes \citep{NJG17,GSP19}, we 
develop a computationally effective algorithm for computing optimal transport distances between Markov chains below.

\subsection{Formal definition: Mirror Sinkhorn on the space of occupancy couplings}
Our method is an adaptation of a version of Sinkhorn's algorithm called Mirror Sinkhorn, first proposed and analyzed by 
\citet{BB23}. This method combines mirror-descent-style updates \citep{NY83,BT03} with alternating projections to two 
convex sets whose intersection corresponds to the feasible set we seek to optimize over. In our adaptation, we choose 
the two sets as 
\[
  \BX = \ev{\mu:\ \sum_{y'} \mu(xy,x'y') = \pa{\gamma \sum_{x''y''} \mu(x''y'',xy) + (1-\gamma) \nu_0(xy)} \PPX(x'|x) 
\quad (\forall xy,x')},
\]
that can be seen to be the set of distributions $\mu$ that satisfy both Equations~\eqref{eq:const_flow} 
and~\eqref{eq:const_X}, and
\[
  \BY = \ev{\mu:\ \sum_{x'} \mu(xy,x'y') = \pa{\gamma \sum_{x''y''} \mu(x''y'',xy) + (1-\gamma) \nu_0(xy)} 
\PPY(y'|y) 
\quad (\forall xy,y')}
\]
which is the set of distributions $\mu$ that satisfy both Equations~\eqref{eq:const_flow} 
and~\eqref{eq:const_Y}. Naturally, the intersection of the two sets corresponds to valid occupancy couplings. 
It remains to define an appropriate notion of entropy for the purpose of regularization. Following \citet{NJG17}, we 
will use the \emph{conditional relative entropy} defined between two joint distributions $\mu,\mu'\in\Delta_{\X\Y\X\Y}$ 
as 
\begin{align*}
 \HH{\mu}{\mu'} &= \sum_{xy,x'y'} \mu(xy,x'y') \log \frac{\mu(xy,x'y') / \sum_{x''y''} 
\mu(xy,x''y'')}{\mu'(xy,x'y') / \sum_{x''y''} \mu'(xy,x''y'')}
\\
&= \sum_{xy,x'y'} \mu(xy,x'y') \log \frac{\pi_\mu(x'y'|xy)}{\pi_{\mu'}(x'y'|xy)}.
\end{align*}
It is an easy exercise to show that $\mathcal{H}$ is a Bregman divergence that is convex in its first argument 
(see, e.g., Appendix A.1~of \citealp{NJG17}). Note however that $\HH{\mu}{\mu'}$ can be zero even if $\mu\neq\mu'$ and 
thus it is not strongly convex in $\mu$.

With these ingredients, we symbolically define our algorithm as calculating the sequence of updates 
\begin{equation}\label{eq:MDupdate}
 \mu_{k+1} = \argmin_{\mu \in \BB_k} \ev{\iprod{\mu}{c} + \frac {1}{\eta} \HH{\mu}{\mu_k}}
\end{equation}
for each $k=1,\dots,K-1$, where $\mu_1$ is the occupancy coupling associated to the trivial coupling $\pi_1(\cdot|xy)=\PPX(\cdot|x)\otimes\PPY(\cdot|y)$ for each state pair $xy\in\X\Y$, $\eta >0 $ is a stepsize (or learning-rate) parameter and $\BB_k$ is chosen to be 
$\BX$ in odd rounds and $\BY$ in even rounds. %\todoG{TODO: check for consistency of odd/even across the paper} 
By adapting tools from the theory of entropy-regularized Markov decision processes, the updates can be computed in 
closed form by solving a system of equations closely resembling the regularized Bellman equations. In particular, 
we define the \emph{Bellman--Sinkhorn operators} for a given transition coupling $\pi$ as the operators
$\TTpi_\X:\real^{\X\Y\times\X}\ra\real^{\X\Y\times\X}$ and $\TTpi_\Y:\real^{\X\Y\times\Y}\ra\real^{\X\Y\times\Y}$ acting on functions 
$\VX \in\real^{\X\Y\times\X}$ and $\VY \in\real^{\X\Y\times\Y}$ respectively as
\begin{equation*}
 \pa{\TTpi_\X \VX}(xy,x') = -\frac{1}{\eta} \log \sum_{y'} \frac{\pi(x'y'|xy)}{\PPX(x'|x)} \exp\pa{-\eta \pa{c(xy) + 
\gamma \sum_{x''} \PPX(x''|x') \VX(x'y',x'')}},
    %\label{eq:Bellman_X}
\end{equation*}
and 
\begin{equation*}
 \pa{\TTpi_\Y \VY}(xy,y') = -\frac{1}{\eta} \log \sum_{x'} \frac{\pi(x'y'|xy)}{\PPY(y'|y)} \exp\pa{-\eta \pa{c(xy) + 
\gamma 
\sum_{y''} \PPY(y''|y') \VY(x'y',y'')}}.
    %\label{eq:Bellman_Y}
\end{equation*}
Then, for odd rounds, the updates can be calculated by solving the fixed-point equations $V_k = \TT^{\pi_k}_\X V_k$,
defining the shorthand $Q_k(xy,x'y') = c(xy) + \gamma \sum_{x''} \PPX(x''|x') V_k(x'y',x'')$, and subsequently 
updating the transition coupling $\pi_k$ multiplicatively as
\begin{equation}
\label{eq:update}
     \pi_{k+1}(x'y'|xy) = \frac{\pi_{k}(x'y'|xy) \exp\pa{-\eta Q_k(xy, x'y')}}{\sum _{y''} \pi_{k}(x'y''|xy) 
\exp\pa{-\eta Q_k(xy, x'y'')}}\PPX(x'|x).
\end{equation}
It is easy to verify that this transition coupling satisfies $\sum_{y'}\pi_{k+1}(x'y'|xy) = P(x'|x)$.
The updates for even rounds are computed analogously with the roles of $\X$ and $\Y$ swapped. We respectively refer to 
the fixed-point equations $V_k = \TT^{\pi_k}_\X V_k$ and $V_k = 
\TT^{\pi_k}_\Y V_k$ as the \emph{Bellman--Sinkhorn equations} for $\MX$ and $\MY$, and the functions $V_k$ and $Q_k$ as 
\emph{value functions}. The following proposition (proved in Appendix~\ref{sec:MDequivalence}) formally 
establishes the equivalence between the two update rules.
% 
% We provide the full proof of the equivalence of the above update rule with 
% Equation~\eqref{eq:MDupdate} in Section~\ref{sec:update_equivalence}. \grg{TODO: write this section down properly.}
\begin{prop} \label{prop:MDequivalence}
Let $\mu_{k+1}$ and $\pi_{k+1}$ be specified for each $k$ as in Equations~\eqref{eq:MDupdate} and~\eqref{eq:update}, 
respectively. Then, $\mu_{k+1} = \mu^{\pi_{k+1}}$ holds for all $k$. 
% Consider the optimization problem defined by Equation~\ref{eq:MDupdate}.
% Suppose the alternating projections method is 
% applied to the sets $\BX$ and $\BY$ as in Equations \ref{eq:Bellman_X} and \ref{eq:Bellman_Y}, with updates governed by Equation 
% (\ref{eq:update}). Then, the solution obtained by this procedure converges to a solution of the optimization problem 
% considered.
\end{prop}

% % \grg{Perhaps it would be better to just present the practical implementation, and then explain later that it is 
% actually a mirror descent method...? But then it will not be so clear why we bothered to talk so much about linear 
% programming previously.}

\subsection{Practical implementation}
\begin{wrapfigure}{r}{0.5\textwidth}
\vspace{-.5cm}
\begin{minipage}{0.5\textwidth}
\begin{algorithm}[H]
\caption{Sinkhorn Value Iteration}\label{alg:main}
\textbf{Input: } $\PPX$, $\PPY$, $c$, $\eta$, $\gamma$, $K$, $m$\\
\textbf{Initialise: }  $\pi_{1} \gets \PPX \otimes \PPY$\;
\For {$k=1,...,K-1$}{
	\eIf {$k$ is odd}{
		$\VX \gets \pa{\TT^{\pi_k}_\X}^{m} \VX$\;   \vspace{-1.1em}\Comment{$\BB_\X$ projection.}
	}
	{
		$\VY \gets \pa{\TT^{\pi_k}_\Y}^{m} \VY$\;   \vspace{-1.1em}\Comment{$\BB_\Y$ projection.}
	}
	$\pi_{k+1} \gets \textbf{update}(\pi_{k})$\;	\vspace{-1.1em}\Comment{Equation \ref{eq:update}} 
}
$\bmu_K \gets \frac{1}{K}\sum _{k=1}^{K}(\mu^{\pi_k})$\;
$\piout \gets \textbf{round}(\pi_{\bmu_K})$\;
$V^\piout \gets \textbf{evaluate}(\piout)$\;
\textbf{Output:} $\piout$, $V^\piout$ \Comment{Final coupling}
\end{algorithm}
\end{minipage}
\vspace{-.4cm}
\end{wrapfigure}
The algorithm described above can be seen as performing online Mirror Sinkhorn updates in each state pair $xy$ 
with a sequence of cost functions $Q_k$, which are computed via solving the Bellman--Sinkhorn equations. Since 
$\TTpi_\X$ 
and $\TTpi_Y$ are easily seen to be contractive with respect to the supremum norm with contraction factor $\gamma$ 
(as shown by a standard calculation included in Appendix~\ref{sec:contraction}), these equations can be solved by an 
adaptation of the classic Value Iteration method of \citet{Bel57}. Concretely, we repeatedly apply 
the Bellman--Sinkhorn operators until the fixed point is reached up to sufficient precision (controlled by the number 
of update steps $m$). We call the resulting method \emph{Sinkhorn Value Iteration} (\SVI), and provide its pseudocode 
as 
Algorithm~\ref{alg:main}.

Notably, while \SVI is defined in its abstract form as a sequence of updates in the space of occupancy 
couplings $\mu_k$, its implementation only works with transition couplings $\pi_k$. 
% This useful (and perhaps surprising) property is a consequence of using conditional-entropy regularization; other 
% forms of regularization would have resulted in less tractable updates on the space of occupancies 
% \citep{NJG17,BCKN21}. 
The final output of \SVI is a transition coupling $\piout$, obtained by computing the average $\bmu_K = \frac{1}{K} 
\sum_{k=1}^K \mu^{\pi_k}$ of all occupancy couplings, computing $\pi_{\bmu_K}$ and then rounding the result 
to a valid transition coupling. In particular we apply a simple rounding procedure due to \citet{altschuler2017near} 
individually on $\pi_{\bmu_K}(\cdot|xy)$ for each state-pair $xy$---for the full details, see 
Appendix~\ref{app:rounding}. 
Besides $\piout$, \SVI also outputs an estimate of $V^*$ in the form of the \emph{value function} $V^{\piout}$, 
as defined in Equation~\eqref{eq:value_function} in Appendix~\ref{app:mdp_values}. This function can be computed 
efficiently by solving the linear system of Bellman equations $V^\piout(xy) = c(xy) + \gamma \sum_{x'y'} 
\piout(x'y'|xy) V^\piout(x'y')$.

A number of small simplifying steps can be made to make the algorithm easier to implement. First, instead of obtaining 
$\pi_{\bmu_K}$ via the computationally expensive procedure described above, one can simply run the rounding procedure on 
the final transition coupling $\pi_K$ and output the result. Second, while theoretical analysis suggests setting 
$m=\infty$ in order to make sure that all projection steps are perfect, such exact computation may be unnecessary and 
inefficient in 
practice, and thus (much) smaller values can be used instead. Third, for small values of $\eta$ the softmax function 
used in the definition of the Bellman--Sinkhorn operator can be accurately approximated by an average with respect to 
$\pi_k(x'y'|xy)/\PPX(x'|x)$, which suggests a simple alternative to the projection steps. This approximates \SVI 
similarly as to how the Mirror Descent Modified Policy Iteration method of \citet{GSP19} approximates the mirror descent 
method of \citet{NJG17} (see also \citealp{AGK12}). The resulting method (that we refer to as \emph{Sinkhorn Policy 
Iteration}, or \SPI) is presented in detail along with its theoretical analysis in Appendix~\ref{app:SPI}. We study 
effects of these implementation choices via a sequence of experiments in Section~\ref{sec:exps}.

% \grg{TODO: add more implementation details: 
% \begin{itemize}
%  \item \sout{outputting a random iterate vs.~the final iterate vs.~the best iterate,}
%  \item the role of evaluation steps $m$ and its relation to modified policy iteration,
%  \item memory and computational complexity,
%  \item rounding procedure.
% \end{itemize}
% }
% 
% \ser{Note on the role of evaluation steps $m$ and its relation to modified policy iteration:
% So far, for the needs of the theoretical analysis, we have just considered the case where we calculate the functions 
% $\VX$ and $\VY$ exactly,  i.e., we apply the operator $\VX = (\TTpi_\X)^m  \VX$ and $\VY = (\TTpi_\Y)^m  \VY$  where $m=\infty$. 
% However, from a practical standpoint such exact computation may be unnecessary and inefficient.
% We can empirically observe that performing an approximate computation of these functions results in a more 
% efficient algorithm that still converges to the desired solution.
% The impact of this parameter $m$ has been extensively studied in the area of reinforcement learning. Two well-known 
% dynamic programming algorithms, Value Iteration (VI) and Policy Iteration (PI), solve the problem for the values $m = 
% 1$ and $m=\infty$ respectively. Modified Policy Iteration \citep{puterman78} generalizes these algorithms by taking any 
% value of m, aiming for less computation per iteration than PI while enjoying the faster convergence of the PI 
% algorithm.}

\subsection{Convergence guarantees}
The following theorem establishes a guarantee on the number of iterations 
necessary for $V^\piout(xy)$ be an $\varepsilon$-accurate approximation of the transport cost 
$\mathbb{W}_\gamma(\MX,\MY;c,x,y)$ for any $x,y$. 
\begin{theorem}\label{thm:main}
Suppose that Sinkhorn Value Iteration is run for $K$ steps with 
regularization parameter $ \eta = \frac 1 {4\infnorm{c}} \sqrt{\frac{(1-\gamma)^3\log |\X||\Y|}{K}}$, and 
initialized with the uniform coupling defined for each $xy,x'y'$ as $\pi_1(x'y'|xy) = \frac{1}{|\X||\Y|}$. Then, for 
any $x_0y_0\in\X\Y$, the output satisfies 
$V^\piout(x_0y_0) \le \mathbb{W}_\gamma(\MX,\MY;c,x_0,y_0) + \varepsilon$ if the number of iterations is at least
\[
K \geq  \frac{324 \infnorm{c}^2 \log |\X||\Y|}{(1-\gamma)^5 \varepsilon^2}.
% K \ge \frac{\HH{\mu^*}{\mu_1}}{(1-\gamma)^2\varepsilon^2}....
\]
\end{theorem}
The proof is relegated to Section~\ref{app:main}, and we present a similar performance guarantee for \SPI in 
Appendix~\ref{app:SPI}. Importantly, these guarantees technically only 
hold when setting $m=\infty$, which is a limitation we discuss in more detail in Section~\ref{sec:discussion}. The 
condition that $\pi_1$ is chosen as the uniform coupling is not necessary and simply made to make the statement easier 
to state. A more detailed statement of the bound is provided in Appendix~\ref{app:proof_final_steps}.
% We conjecture that guarantees similar to Theorem~\ref{thm:main} can be shown 
% for both algorithms for finite values of $m$ by adapting the techniques of \citet{GSP19} and \citet{MN23} developed for 
% analyzing regularized dynamic programming algorithms for Markov decision processes, but we leave this question open for 
% future work.

\section{Experiments}\label{sec:exps}
We have conducted a range of experiments on some simple environments with the purpose of illustrating the numerical 
properties of our algorithms and some aspects of the distance metrics we studied. 
Due to space restrictions, we only report a very limited subsample of the results below, and refer the reader to 
Appendix~\ref{app:experiments} for the complete suite\footnote{The code is available at 
\url{https://github.com/SergioCalo/SVI}}. 

One set of experiments we report here addresses the biggest open question left behind our theory: the effect of the 
choice of $m$ on the quality of the updates. For this experiment, we use the classic ``4-rooms'' environment first 
studied by \cite{sutton1999}, and run both SVI (Algorithm~\ref{alg:main}) and SPI (Algorithm~\ref{alg:SPI}) for a range of 
different choices of $m$, and a fixed $\gamma=0.95$. The results of this study are shown in Figure \ref{fig:role_m}. The plots indicate that 
the estimates produced by both algorithms converge towards the true distance at a rate that is basically unaffected by 
$m$, and in particular even a value of $m=1$ remains competitive. This observation is consistent 
across all of our experiments. Also, the output of SPI appears to converge slightly more slowly towards the optimum in 
this experiment, but this observation is not entirely consistent and can be likely ascribed to the fact that the 
learning rate was not optimized to favor either algorithm in this experiment. In most experiments, the two algorithms 
performed very similarly, up to some small occasional differences. 

We have also conducted a number of experiments to illustrate the potential of optimal-transport distances for comparing 
Markov chains of different sizes and transition functions. In the experiment we show here, we compare two Markov chains 
illustrated in Figure~\ref{fig:agg}. The first Markov 
chain $M_\X$ is a simple, nine-state ``gridworld'' environment, which has its initial state in the upper left corner 
(denoted as $s_0$) and a reward of $+1$ in the lower left corner (shown in blue). The second, $M_\Y$, is an instance of 
the 4-rooms environment, where each room is a rotation of the aforementioned small grid. The transition kernels in both 
environments are uniform distributions over the adjacent cells in the four principal directions. The plot shows the 
distances between the two chains as a function of the initial state of $M_\Y$, revealing an intuitive pattern of 
similarities that captures the symmetries of $M_\Y$.

\begin{figure}
     \centering
     \begin{subfigure}[b]{0.45\textwidth}
         \centering
         \includegraphics[width=\textwidth]{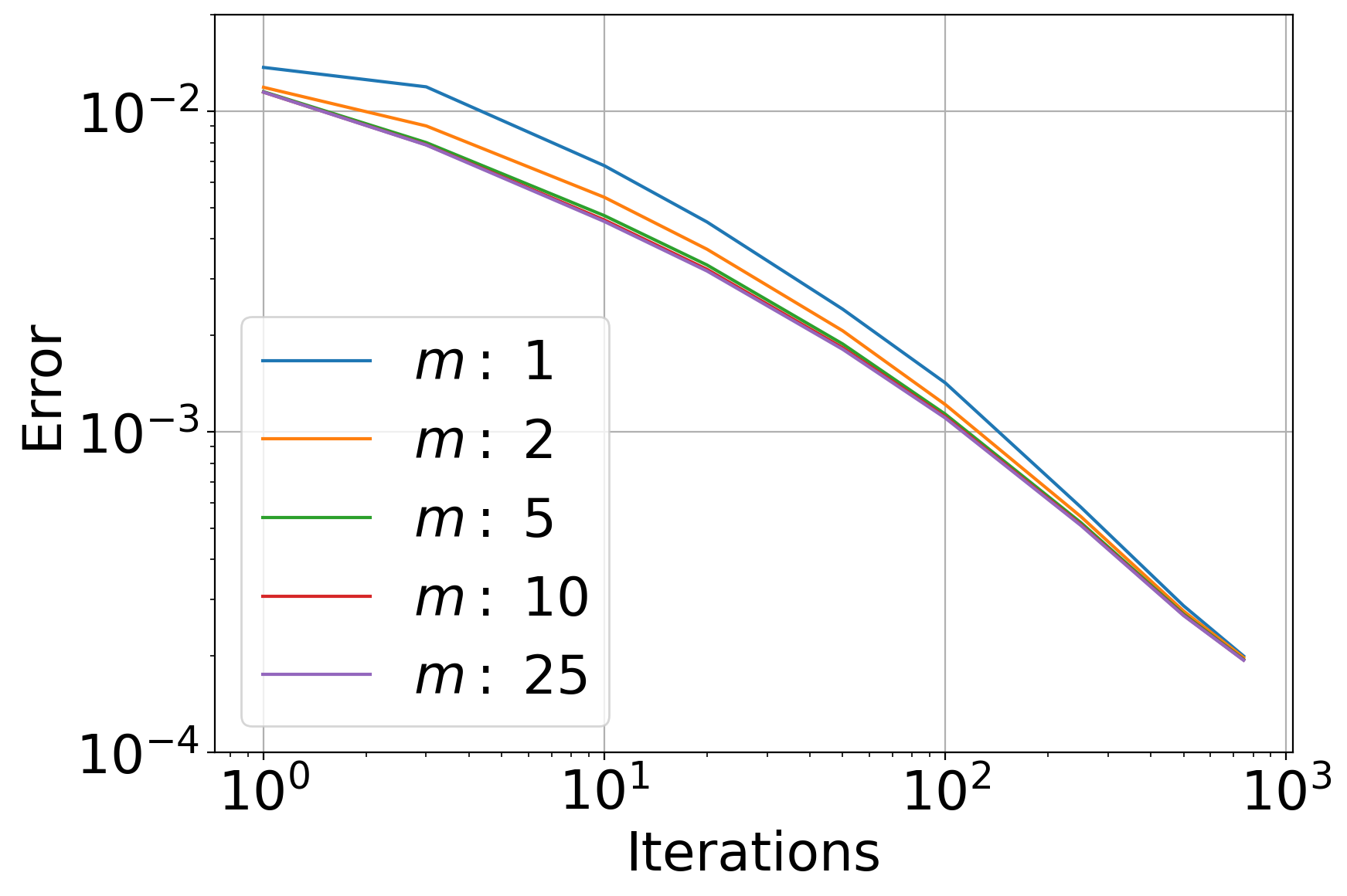}
         \caption{SVI (Algorithm~\ref{alg:main})}
         \label{fig:m_alg1}
     \end{subfigure}
     \hfill
     \begin{subfigure}[b]{0.45\textwidth}
         \centering
         \includegraphics[width=\textwidth]{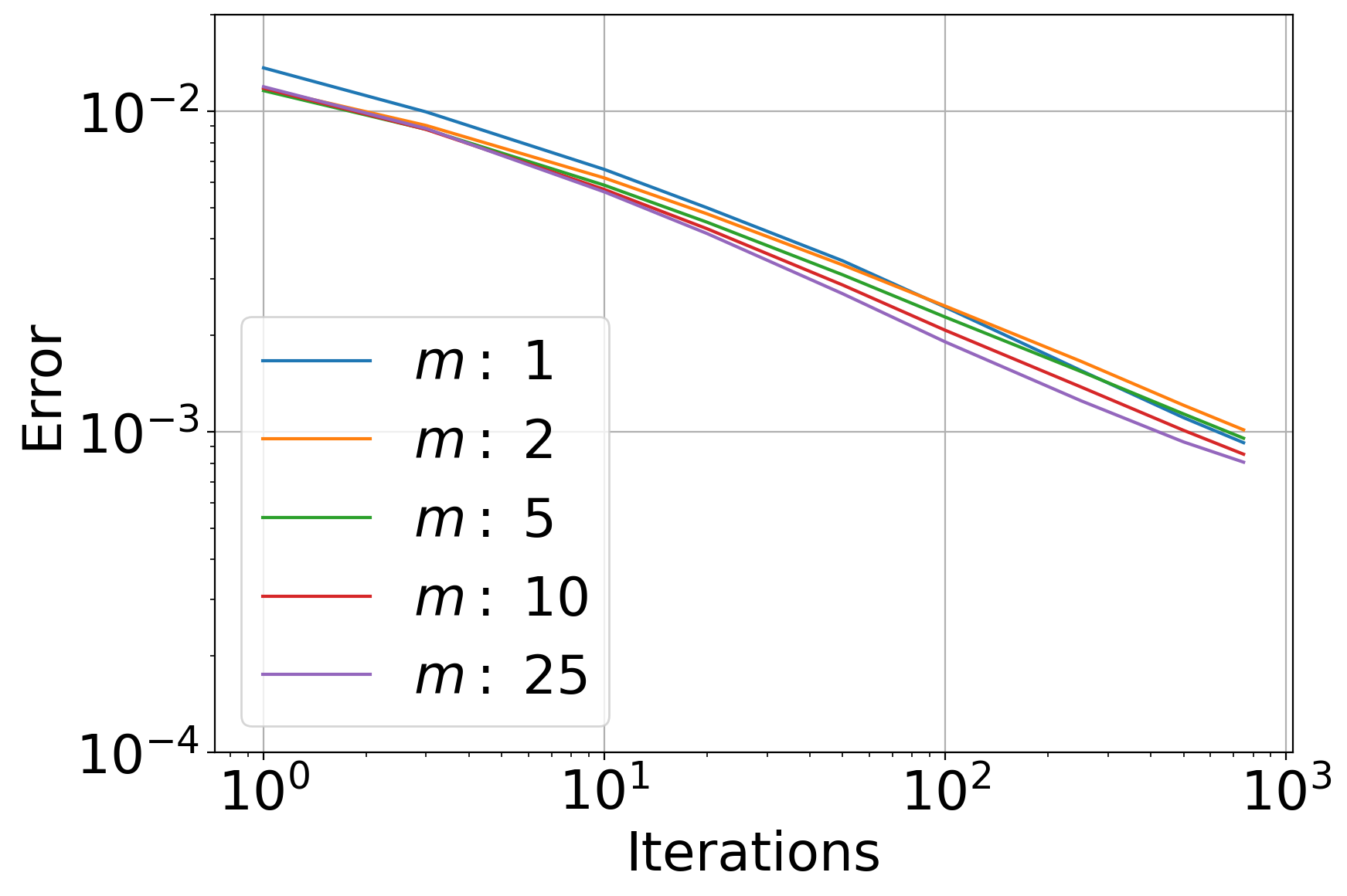}
         \caption{SPI (Algorithm~\ref{alg:SPI})}
         \label{fig:m_alg2}
     \end{subfigure}
        \caption{Estimated transport cost as a function $k$, for various choices of $m$ and $\eta = 1$.
}
        \label{fig:role_m}
\end{figure}

\begin{figure}[h!]
	\center 
	\includegraphics[width=0.8\textwidth]{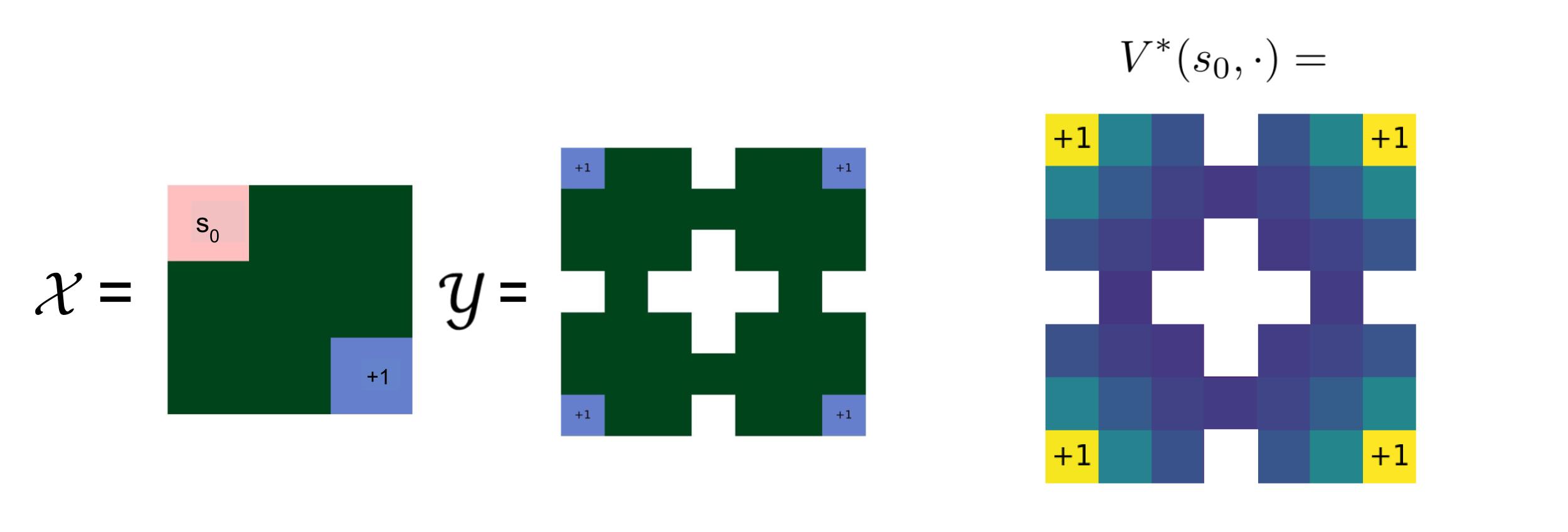} 
	\caption{Visual representation of the distances computed between the chains $M_\X$ and $M_\Y$.}
	\label{fig:agg}
\end{figure}

\section{Discussion}
\label{sec:discussion}
\vspace{-2mm}
We discuss some further aspects of our framework and results below.

\vspace{-2mm}

\paragraph{Representation learning for reinforcement learning.} Among the numerous applications listed in 
Section~\ref{sec:intro} and 
Appendix~\ref{app:related_work}, the most interesting for us is using our metrics for representation learning in 
RL. As mentioned earlier, bisimulation metrics have been extensively used for this purpose in the 
past.
In particular, 
almost all such work uses bisimulation metrics to compare states within the same MDP and use the resulting similarity 
metrics for merging states that are at low distance (an approach called ``state aggregation''). As our results 
highlight, this is a rather
narrow view of what bisimulation metrics are capable of: they can define
similarity metrics between processes that live on potentially different state spaces, which in particular can be used 
to select representations by minimizing the distance between a high-dimensional process and a set of low-dimensional 
representations. Curiously, our LP formulation may allow differentiating the distances with respect to the transition 
kernels, which we believe will be an important property for future developments in representation learning for 
RL.\looseness=-1

\vspace{-2mm}
\paragraph{Limitations of the theory.} In their current form, our theoretical guarantees in Theorems~\ref{thm:main} 
and~\ref{thm:main2} only apply to perfect projection and evaluation steps, corresponding to setting $m=\infty$. 

We conjecture that this limitation can be addressed with a more careful analysis, and results similar to those of 
Theorems~\ref{thm:main} and~\ref{thm:main2} can be shown, potentially at the price of a worse dependence on the 
effective horizon $1/(1-\gamma)$ \citep{SGGG12,SGGLG15}, by making use of the techniques of \citet{GSP19} and 
\citet{MN23} for analyzing regularized dynamic-programming algorithms. 

\vspace{-2mm}
\paragraph{From dynamic programming to learning from data.} This paper focuses on computing distances between known 
Markov processes via dynamic-programming-style methods. In the most interesting applications however, the transition 
kernels are unknown, which requires the development of new tools. We are confident that our framework can serve as a 
solid basis for such developments, and in particular that one can port many ideas from the field of reinforcement 
learning that is essentially all about turning dynamic-programming methods into algorithms that can learn from 
interaction data. 
Additionally,
we believe that our LP formulation in Section~\ref{sec:LP} makes it much easier to import further ideas from 
computational optimal transport, and in particular that stochastic optimization methods like those of \citet{GCPB16} 
can be adapted to solving our linear programs.

\section*{Acknowledgments.} G.~Neu would like to thank Marin Ballu and Quentin Berthet for clarifying a number of 
details of their analysis of Mirror Sinkhorn, Tristan Brug\`ere for pointing us to their publicly available code, and 
Anna Korba and Stefan Schrott for kindly helping out with some references about optimal transport between stochastic 
processes.

\bibliographystyle{abbrvnat}
\bibliography{OT_VI_notes,ngbib}

\begin{thebibliography}{88}
\providecommand{\natexlab}[1]{#1}
\providecommand{\url}[1]{\texttt{#1}}
\expandafter\ifx\csname urlstyle\endcsname\relax
  \providecommand{\doi}[1]{doi: #1}\else
  \providecommand{\doi}{doi: \begingroup \urlstyle{rm}\Url}\fi

\bibitem[Abate(2013)]{Aba13}
A.~Abate.
\newblock Approximation metrics based on probabilistic bisimulations for
  general state-space {Markov} processes: a survey.
\newblock \emph{Electronic Notes in Theoretical Computer Science},
  297:\penalty0 3--25, 2013.

\bibitem[Abbasi-Yadkori et~al.(2019)Abbasi-Yadkori, Bartlett, Bhatia, Lazic,
  Szepesvari, and Weisz]{LABWBS19}
Y.~Abbasi-Yadkori, P.~Bartlett, K.~Bhatia, N.~Lazic, C.~Szepesvari, and
  G.~Weisz.
\newblock {POLITEX}: Regret bounds for policy iteration using expert
  prediction.
\newblock In \emph{International Conference on Machine Learning (ICML)}, pages
  3692--3702, 2019.

\bibitem[Aczel(1988)]{Aczel88}
P.~Aczel.
\newblock \emph{Non-well-founded sets}.
\newblock CSLI lecture notes, no.~14, 1988.

\bibitem[Agarwal et~al.(2021{\natexlab{a}})Agarwal, Kakade, Lee, and
  Mahajan]{AKLM21}
A.~Agarwal, S.~M. Kakade, J.~D. Lee, and G.~Mahajan.
\newblock On the theory of policy gradient methods: Optimality, approximation,
  and distribution shift.
\newblock \emph{Journal of Machine Learning Research}, 22\penalty0
  (98):\penalty0 1--76, 2021{\natexlab{a}}.

\bibitem[Agarwal et~al.(2021{\natexlab{b}})Agarwal, Machado, Castro, and
  Bellemare]{agarwal2021}
R.~Agarwal, M.~C. Machado, P.~S. Castro, and M.~G. Bellemare.
\newblock Contrastive behavioral similarity embeddings for generalization in
  reinforcement learning.
\newblock In \emph{International Conference on Learning Representations
  (ICLR)}, 2021{\natexlab{b}}.

\bibitem[Altschuler et~al.(2017)Altschuler, Niles-Weed, and
  Rigollet]{altschuler2017near}
J.~Altschuler, J.~Niles-Weed, and P.~Rigollet.
\newblock Near-linear time approximation algorithms for optimal transport via
  {S}inkhorn iteration.
\newblock \emph{Advances in Neural Information Processing Systems}, 30, 2017.

\bibitem[Arjovsky et~al.(2017)Arjovsky, Chintala, and Bottou]{ACB17}
M.~Arjovsky, S.~Chintala, and L.~Bottou.
\newblock Wasserstein generative adversarial networks.
\newblock In \emph{International Conference on Machine Learning (ICML)}, pages
  214--223, 2017.

\bibitem[Azar et~al.(2012)Azar, G{\'o}mez, and Kappen]{AGK12}
M.~G. Azar, V.~G{\'o}mez, and H.~J. Kappen.
\newblock Dynamic policy programming.
\newblock \emph{Journal of Machine Learning Research}, 13\penalty0
  (Nov):\penalty0 3207--3245, 2012.

\bibitem[Backhoff-Veraguas et~al.(2017)Backhoff-Veraguas, Beiglbock, Lin, and
  Zalashko]{BBLZ17}
J.~Backhoff-Veraguas, M.~Beiglbock, Y.~Lin, and A.~Zalashko.
\newblock Causal transport in discrete time and applications.
\newblock \emph{SIAM Journal on Optimization}, 27\penalty0 (4):\penalty0
  2528--2562, 2017.

\bibitem[Backhoff-Veraguas et~al.(2020)Backhoff-Veraguas, Bartl, Beiglb{\"o}ck,
  and Eder]{BBBE20}
J.~Backhoff-Veraguas, D.~Bartl, M.~Beiglb{\"o}ck, and M.~Eder.
\newblock All adapted topologies are equal.
\newblock \emph{Probability Theory and Related Fields}, 178:\penalty0
  1125--1172, 2020.

\bibitem[Ballu and Berthet(2023)]{BB23}
M.~Ballu and Q.~Berthet.
\newblock Mirror {Sinkhorn}: Fast online optimization on transport polytopes.
\newblock In \emph{International Conference on Machine Learning (ICML)}, 2023.

\bibitem[Bartl and Wiesel(2022)]{BW22}
D.~Bartl and J.~Wiesel.
\newblock Sensitivity of multiperiod optimization problems in adapted
  wasserstein distance.
\newblock \emph{arXiv preprint arXiv:2208.05656}, 2022.

\bibitem[Bayraktar and Han(2023)]{BH23}
E.~Bayraktar and B.~Han.
\newblock Fitted value iteration methods for bicausal optimal transport.
\newblock \emph{arXiv preprint arXiv:2306.12658}, 2023.

\bibitem[Beck and Teboulle(2003)]{BT03}
A.~Beck and M.~Teboulle.
\newblock Mirror descent and nonlinear projected subgradient methods for convex
  optimization.
\newblock \emph{Operations Research Letters}, 31\penalty0 (3):\penalty0
  167--175, 2003.

\bibitem[Bellman(1957)]{Bel57}
R.~Bellman.
\newblock \emph{Dynamic Programming}.
\newblock Princeton University Press, Princeton, New Jersey, 1957.

\bibitem[Bertsekas(1977)]{Berts77}
D.~P. Bertsekas.
\newblock Monotone mappings with application in dynamic programming.
\newblock \emph{SIAM Journal on Control and Optimization}, 15\penalty0
  (3):\penalty0 438--464, 1977.
\newblock \doi{10.1137/0315031}.

\bibitem[Bian and Abate(2017)]{BA17}
G.~Bian and A.~Abate.
\newblock On the relationship between bisimulation and trace equivalence in an
  approximate probabilistic context.
\newblock In \emph{Foundations of Software Science and Computation Structures
  (FoSSaCS)}, pages 321--337, 2017.

\bibitem[Brug{\`e}re et~al.(2024)Brug{\`e}re, Wan, and Wang]{BWW24}
T.~Brug{\`e}re, Z.~Wan, and Y.~Wang.
\newblock Distances for {M}arkov chains, and their differentiation.
\newblock In \emph{International Conference on Algorithmic Learning Theory
  (ALT)}, pages 282--336, 2024.

\bibitem[Cao(1999)]{Cao99}
X.-R. Cao.
\newblock Single sample path-based optimization of {Markov} chains.
\newblock \emph{Journal of optimization theory and applications}, 100:\penalty0
  527--548, 1999.

\bibitem[Castro(2020)]{castro2020scalable}
P.~S. Castro.
\newblock Scalable methods for computing state similarity in deterministic
  {M}arkov decision processes.
\newblock In \emph{AAAI Conference on Artificial Intelligence (AAAI)}, pages
  10069--10076, 2020.

\bibitem[Castro et~al.(2022)Castro, Kastner, Panangaden, and
  Rowland]{castro2022kernel}
P.~S. Castro, T.~Kastner, P.~Panangaden, and M.~Rowland.
\newblock A kernel perspective on behavioural metrics for {Markov} decision
  processes.
\newblock \emph{Transactions on Machine Learning Research}, 2022.

\bibitem[Chen et~al.(2012)Chen, van Breugel, and Worrell]{CvBW12}
D.~Chen, F.~van Breugel, and J.~Worrell.
\newblock On the complexity of computing probabilistic bisimilarity.
\newblock In \emph{Foundations of Software Science and Computation Structure
  (FoSSaCS)}, pages 437--451, 2012.

\bibitem[Chen et~al.(2022)Chen, Lim, M{\'e}moli, Wan, and Wang]{CLMWW22}
S.~Chen, S.~Lim, F.~M{\'e}moli, Z.~Wan, and Y.~Wang.
\newblock {Weisfeiler--Lehman} meets {Gromov--Wasserstein}.
\newblock In \emph{International Conference on Machine Learning (ICML)}, pages
  3371--3416, 2022.

\bibitem[Chuang and Jegelka(2022)]{CJ22}
C.-Y. Chuang and S.~Jegelka.
\newblock Tree mover's distance: Bridging graph metrics and stability of graph
  neural networks.
\newblock \emph{Advances in Neural Information Processing Systems},
  35:\penalty0 2944--2957, 2022.

\bibitem[Courty et~al.(2016)Courty, Flamary, Tuia, and Rakotomamonjy]{CFTR16}
N.~Courty, R.~Flamary, D.~Tuia, and A.~Rakotomamonjy.
\newblock Optimal transport for domain adaptation.
\newblock \emph{IEEE transactions on pattern analysis and machine
  intelligence}, 39\penalty0 (9):\penalty0 1853--1865, 2016.

\bibitem[Courty et~al.(2018)Courty, Flamary, and Ducoffe]{CFD18}
N.~Courty, R.~Flamary, and M.~Ducoffe.
\newblock Learning {Wasserstein} embeddings.
\newblock In \emph{International Conference on Learning Representations
  (ICLR)}, pages 1--13, 2018.

\bibitem[Cuturi(2013)]{cuturi2013sinkhorn}
M.~Cuturi.
\newblock Sinkhorn distances: Lightspeed computation of optimal transport.
\newblock \emph{Advances in Neural Information Processing Systems}, 26, 2013.

\bibitem[de~Ghellinck(1960)]{Ghe60}
G.~de~Ghellinck.
\newblock Les probl\`emes de d\'ecisions s\'equentielles.
\newblock \emph{Cahiers du Centre d’\'Etudes de Recherche Op\'erationnelle},
  2:\penalty0 161--179, 1960.

\bibitem[Denardo(1970)]{Den70}
E.~V. Denardo.
\newblock On linear programming in a {Markov} decision problem.
\newblock \emph{Management Science}, 16\penalty0 (5):\penalty0 281--288, 1970.

\bibitem[d'Epenoux(1963)]{dEp63}
F.~d'Epenoux.
\newblock A probabilistic production and inventory problem.
\newblock \emph{Management Science}, 10\penalty0 (1):\penalty0 98--108, 1963.

\bibitem[Desharnais et~al.(1999)Desharnais, Gupta, Jagadeesan, and
  Panangaden]{DGJP99}
J.~Desharnais, V.~Gupta, R.~Jagadeesan, and P.~Panangaden.
\newblock Metrics for labeled {Markov} systems.
\newblock In \emph{International Conference on Concurrency Theory (CONCUR)},
  pages 258--273, 1999.

\bibitem[Desharnais et~al.(2002)Desharnais, Jagadeesan, Gupta, and
  Panangaden]{desharnais2002metric}
J.~Desharnais, R.~Jagadeesan, V.~Gupta, and P.~Panangaden.
\newblock The metric analogue of weak bisimulation for probabilistic processes.
\newblock In \emph{IEEE Symposium on Logic in Computer Science (LICS)}, pages
  413--422, 2002.

\bibitem[Desharnais et~al.(2004)Desharnais, Gupta, Jagadeesan, and
  Panangaden]{DGJP04}
J.~Desharnais, V.~Gupta, R.~Jagadeesan, and P.~Panangaden.
\newblock Metrics for labelled {Markov processes}.
\newblock \emph{Theoretical Computer Science}, 318\penalty0 (3):\penalty0
  323--354, 2004.

\bibitem[Eckstein and Pammer(2024)]{EP24}
S.~Eckstein and G.~Pammer.
\newblock Computational methods for adapted optimal transport.
\newblock \emph{The Annals of Applied Probability}, 34\penalty0 (1A):\penalty0
  675--713, 2024.

\bibitem[Even-Dar et~al.(2009)Even-Dar, Kakade, and
  Mansour]{even-dar09OnlineMDP}
E.~Even-Dar, S.~M. Kakade, and Y.~Mansour.
\newblock Online {M}arkov decision processes.
\newblock \emph{Mathematics of Operations Research}, 34\penalty0 (3):\penalty0
  726--736, 2009.

\bibitem[Ferns et~al.(2004)Ferns, Panangaden, and Precup]{FPP04}
N.~Ferns, P.~Panangaden, and D.~Precup.
\newblock Metrics for finite {Markov} decision processes.
\newblock In \emph{Uncertainty in Artificial Intelligence (UAI)}, pages
  162--169, 2004.

\bibitem[Ferns et~al.(2006)Ferns, Castro, Precup, and Panangaden]{FCPP06}
N.~Ferns, P.~S. Castro, D.~Precup, and P.~Panangaden.
\newblock Methods for computing state similarity in {Markov} decision
  processes.
\newblock In \emph{Uncertainty in Artificial Intelligence (UAI)}, pages
  174--181, 2006.

\bibitem[Forti and Honsell(1983)]{FortiH83}
M.~Forti and F.~Honsell.
\newblock Set theory with free construction principles.
\newblock \emph{Annali Scuola Normale Superiore, Pisa}, Serie IV X\penalty0
  (3):\penalty0 493--522, 1983.

\bibitem[Galichon(2016)]{galichon2016}
A.~Galichon.
\newblock \emph{Optimal Transport Methods in Economics}.
\newblock Princeton University Press, 2016.

\bibitem[Geist et~al.(2019)Geist, Scherrer, and Pietquin]{GSP19}
M.~Geist, B.~Scherrer, and O.~Pietquin.
\newblock A theory of regularized {Markov} decision processes.
\newblock In \emph{International Conference on Machine Learning (ICML)}, pages
  2160--2169, 2019.

\bibitem[Gelada et~al.(2019)Gelada, Kumar, Buckman, Nachum, and
  Bellemare]{gelada2019}
C.~Gelada, S.~Kumar, J.~Buckman, O.~Nachum, and M.~G. Bellemare.
\newblock Learning continuous latent space models for representation learning.
\newblock In \emph{International Conference on Machine Learning (ICML)}, 2019.

\bibitem[Genevay et~al.(2016)Genevay, Cuturi, Peyr{\'e}, and Bach]{GCPB16}
A.~Genevay, M.~Cuturi, G.~Peyr{\'e}, and F.~Bach.
\newblock Stochastic optimization for large-scale optimal transport.
\newblock \emph{Advances in neural information processing systems}, 29, 2016.

\bibitem[Giacalone et~al.(1990)Giacalone, Jou, and Smolka]{Giacalone90}
A.~Giacalone, C.-C. Jou, and S.~A. Smolka.
\newblock Algebraic reasoning for probabilistic concurrent systems.
\newblock In \emph{IFIP Conference on Programming Concepts and Methods}, pages
  443--458, 1990.

\bibitem[Givan et~al.(2003)Givan, Dean, and Greig]{Givan93}
R.~Givan, T.~Dean, and M.~Greig.
\newblock Equivalence notions and model minimization in {M}arkov decision
  processes.
\newblock \emph{Artificial Intelligence}, 147:\penalty0 163--223, 2003.

\bibitem[Hansen-Estruch et~al.(2022)Hansen-Estruch, Zhang, Nair, Yin, and
  Levine]{hansen2022bisimulation}
P.~Hansen-Estruch, A.~Zhang, A.~Nair, P.~Yin, and S.~Levine.
\newblock Bisimulation makes analogies in goal-conditioned reinforcement
  learning.
\newblock In \emph{International Conference on Machine Learning (ICML)}, pages
  8407--8426, 2022.

\bibitem[Howard(1960)]{How60}
R.~A. Howard.
\newblock \emph{Dynamic programming and {Markov} processes}.
\newblock John Wiley, 1960.

\bibitem[Jiang(2018)]{Jia18}
N.~Jiang.
\newblock Notes on state abstractions, 2018.

\bibitem[Jiang(2024)]{jiang2024}
N.~Jiang.
\newblock A note on loss functions and error compounding in model-based
  reinforcement learning.
\newblock \emph{arXiv preprint arXiv:2404.09946}, 2024.

\bibitem[Jonsson and Larsen(1991)]{Jonsson91}
B.~Jonsson and K.~G. Larsen.
\newblock Specification and refinement of probabilistic processes.
\newblock In \emph{IEEE Symposium on Logic in Computer Science (LICS)}, pages
  167--183, 1991.

\bibitem[Kakade(2001)]{K01}
S.~Kakade.
\newblock A natural policy gradient.
\newblock \emph{Advances in Neural Information Processing Systems},
  14:\penalty0 1531--1538, 2001.

\bibitem[Kakade and Langford(2002)]{KL02}
S.~Kakade and J.~Langford.
\newblock Approximately optimal approximate reinforcement learning.
\newblock In \emph{International Conference on Machine Learning (ICML)}, pages
  267--274, 2002.

\bibitem[Kantorovich(1942)]{Kan42}
L.~V. Kantorovich.
\newblock On the translocation of masses.
\newblock In \emph{Dokl. Akad. Nauk. USSR (NS)}, volume~37, pages 199--201,
  1942.

\bibitem[Kemertas and Jepson(2022{\natexlab{a}})]{kemertas2022}
M.~Kemertas and A.~Jepson.
\newblock Approximate policy iteration with bisimulation metrics.
\newblock \emph{arXiv preprint arXiv:2202.02881}, 2022{\natexlab{a}}.

\bibitem[Kemertas and Jepson(2022{\natexlab{b}})]{kemertas22}
M.~Kemertas and A.~Jepson.
\newblock Approximate policy iteration with bisimulation metrics.
\newblock \emph{Transactions of Machine Learning Research}, 2022{\natexlab{b}}.

\bibitem[Kolouri et~al.(2017)Kolouri, Park, Thorpe, Slepcev, and
  Rohde]{kolouri2017}
S.~Kolouri, S.~R. Park, M.~Thorpe, D.~Slepcev, and G.~K. Rohde.
\newblock Optimal mass transport: signal processing and machine-learning
  applications.
\newblock \emph{IEEE Signal Processing Magazine}, 34\penalty0 (4):\penalty0
  43--59, 2017.

\bibitem[Kruskal(1964)]{kruskal1964}
J.~B. Kruskal.
\newblock Multidimensional scaling by optimizing goodness of fit to a nonmetric
  hypothesis.
\newblock \emph{Psychometrika}, 29\penalty0 (1):\penalty0 1--27, 1964.

\bibitem[Larsen and Skou(1989)]{Larsen1989BisimulationTP}
K.~G. Larsen and A.~Skou.
\newblock Bisimulation through probabilistic testing.
\newblock In \emph{ACM Symposium on Principles of Programming Languages
  (POPL)}, pages 344--352, 1989.

\bibitem[Lassalle(2018)]{lassalle18}
R.~Lassalle.
\newblock \emph{Causal Transport Plans and Their Monge–Kantorovich Problems}.
\newblock Taylor \& Francis, 2018.

\bibitem[Lattimore and Szepesv{\'a}ri(2020)]{lattimore2020bandit}
T.~Lattimore and C.~Szepesv{\'a}ri.
\newblock \emph{Bandit algorithms}.
\newblock Cambridge University Press, 2020.

\bibitem[Manne(1960)]{Man60}
A.~S. Manne.
\newblock Linear programming and sequential decisions.
\newblock \emph{Management Science}, 6\penalty0 (3):\penalty0 259--267, 1960.

\bibitem[Milner(1989)]{Milner89}
R.~Milner.
\newblock \emph{Communication and Concurrency}.
\newblock Prentice Hall, 1989.

\bibitem[Moulin and Neu(2023)]{MN23}
A.~Moulin and G.~Neu.
\newblock Optimistic planning via regularized dynamic programming.
\newblock In \emph{International Conference on Machine Learning (ICML)}, 2023.

\bibitem[Moulos(2021)]{moulos2021bicausal}
V.~Moulos.
\newblock Bicausal optimal transport for {M}arkov chains via dynamic
  programming.
\newblock In \emph{IEEE International Symposium on Information Theory (ISIT)},
  pages 1688--1693, 2021.

\bibitem[Nemirovski and Yudin(1983)]{NY83}
A.~Nemirovski and D.~Yudin.
\newblock \emph{Problem Complexity and Method Efficiency in Optimization}.
\newblock Wiley Interscience, 1983.

\bibitem[Neu et~al.(2017)Neu, Jonsson, and G{\'o}mez]{NJG17}
G.~Neu, A.~Jonsson, and V.~G{\'o}mez.
\newblock A unified view of entropy-regularized {Markov} decision processes.
\newblock \emph{arXiv preprint arXiv:1705.07798}, 2017.

\bibitem[O'Connor et~al.(2022)O'Connor, McGoff, and Nobel]{OMN21}
K.~O'Connor, K.~McGoff, and A.~B. Nobel.
\newblock Optimal transport for stationary {Markov} chains via policy
  iteration.
\newblock \emph{Journal of Machine Learning Research}, 23\penalty0
  (1):\penalty0 2175--2226, 2022.

\bibitem[Park(1981)]{Park81}
D.~M.~R. Park.
\newblock Concurrency and automata on infinite sequences.
\newblock In \emph{GI Symposium on Theoretical Computer Science}, volume 104 of
  \emph{Lecture Notes in Computer Science}, pages 167--183. Springer, 1981.

\bibitem[Peyr{\'e} and Cuturi(2019)]{COTFNT}
G.~Peyr{\'e} and M.~Cuturi.
\newblock Computational optimal transport.
\newblock \emph{Foundations and Trends in Machine Learning}, 11\penalty0
  (5-6):\penalty0 355--607, 2019.

\bibitem[Pflug(2010)]{Pfl10}
G.~C. Pflug.
\newblock Version-independence and nested distributions in multistage
  stochastic optimization.
\newblock \emph{SIAM Journal on Optimization}, 20\penalty0 (3):\penalty0
  1406--1420, 2010.

\bibitem[Pflug and Pichler(2012)]{pflug12}
G.~C. Pflug and A.~Pichler.
\newblock A distance for multistage stochastic optimization models.
\newblock \emph{SIAM Journal on Optimization}, 22\penalty0 (1):\penalty0 1--23,
  2012.

\bibitem[Puterman(1994)]{Puterman1994}
M.~L. Puterman.
\newblock \emph{{M}arkov Decision Processes: Discrete Stochastic Dynamic
  Programming}.
\newblock Wiley-Interscience, April 1994.

\bibitem[Sangiorgi(2009)]{sangiorgi2009origins}
D.~Sangiorgi.
\newblock On the origins of bisimulation and coinduction.
\newblock \emph{ACM Transactions on Programming Languages and Systems
  (TOPLAS)}, 31\penalty0 (4):\penalty0 1--41, 2009.

\bibitem[Scherrer(2013)]{Sch13}
B.~Scherrer.
\newblock On the performance bounds of some policy search dynamic programming
  algorithms.
\newblock \emph{arXiv preprint arXiv:1306.0539}, 2013.

\bibitem[Scherrer et~al.(2012)Scherrer, Gabillon, Ghavamzadeh, and
  Geist]{SGGG12}
B.~Scherrer, V.~Gabillon, M.~Ghavamzadeh, and M.~Geist.
\newblock Approximate modified policy iteration.
\newblock In \emph{International Conference on Machine Learning (ICML)}, pages
  1207--1214, 2012.

\bibitem[Scherrer et~al.(2015)Scherrer, Ghavamzadeh, Gabillon, Lesner, and
  Geist]{SGGLG15}
B.~Scherrer, M.~Ghavamzadeh, V.~Gabillon, B.~Lesner, and M.~Geist.
\newblock Approximate modified policy iteration and its application to the game
  of {T}etris.
\newblock \emph{Journal of Machine Learning Research}, 16:\penalty0 1629--1676,
  2015.

\bibitem[Schiebinger et~al.(2019)Schiebinger, Shu, Tabaka, Cleary, Subramanian,
  Solomon, Gould, Liu, Lin, Berube, Lee, Chen, Brumbaug, Rigollet,
  Hochedlinger, Jaenisch, Regev, and Lander]{Sch+19}
G.~Schiebinger, J.~Shu, M.~Tabaka, B.~Cleary, V.~Subramanian, A.~Solomon,
  J.~Gould, S.~Liu, S.~Lin, P.~Berube, L.~Lee, J.~Chen, J.~Brumbaug,
  P.~Rigollet, K.~Hochedlinger, R.~Jaenisch, A.~Regev, and E.~S. Lander.
\newblock Optimal-transport analysis of single-cell gene expression identifies
  developmental trajectories in reprogramming.
\newblock \emph{Cell}, 176\penalty0 (4):\penalty0 928--943, 2019.

\bibitem[Shi et~al.(2024)Shi, De~Bortoli, Campbell, and Doucet]{SDCD24}
Y.~Shi, V.~De~Bortoli, A.~Campbell, and A.~Doucet.
\newblock Diffusion {S}chr{\"o}dinger bridge matching.
\newblock \emph{Advances in Neural Information Processing Systems}, 36, 2024.

\bibitem[Sinkhorn and Knopp(1967)]{sinkhorn1967concerning}
R.~Sinkhorn and P.~Knopp.
\newblock Concerning nonnegative matrices and doubly stochastic matrices.
\newblock \emph{Pacific Journal of Mathematics}, 21\penalty0 (2):\penalty0
  343--348, 1967.

\bibitem[Song et~al.(2020)Song, Sohl-Dickstein, Kingma, Kumar, Ermon, and
  Poole]{SSKKEP20}
Y.~Song, J.~Sohl-Dickstein, D.~P. Kingma, A.~Kumar, S.~Ermon, and B.~Poole.
\newblock Score-based generative modeling through stochastic differential
  equations.
\newblock \emph{arXiv preprint arXiv:2011.13456}, 2020.

\bibitem[Sutton et~al.(1999)Sutton, Precup, and Singh]{sutton1999}
R.~S. Sutton, D.~Precup, and S.~Singh.
\newblock Between {MDPs} and semi-{MDPs}: A framework for temporal abstraction
  in reinforcement learning.
\newblock \emph{Artificial intelligence}, 112\penalty0 (1-2):\penalty0
  181--211, 1999.

\bibitem[Titouan et~al.(2019)Titouan, Courty, Tavenard, and Flamary]{TCTF19}
V.~Titouan, N.~Courty, R.~Tavenard, and R.~Flamary.
\newblock Optimal transport for structured data with application on graphs.
\newblock In \emph{International Conference on Machine Learning (ICML)}, pages
  6275--6284, 2019.

\bibitem[van Benthem(1983)]{vanBenthem1983-VANMLA}
J.~van Benthem.
\newblock \emph{Modal Logic and Classical Logic}.
\newblock Bibliopolis, 1983.

\bibitem[van Breugel and Worrell(2001)]{vBW01}
F.~van Breugel and J.~Worrell.
\newblock An algorithm for quantitative verification of probabilistic
  transition systems.
\newblock In \emph{International Conference on Concurrency Theory (CONCUR)},
  pages 336--350, 2001.

\bibitem[Villani(2009)]{villani2009optimal}
C.~Villani.
\newblock \emph{Optimal transport: old and new}, volume 338.
\newblock Springer, 2009.

\bibitem[Xu et~al.(2020)Xu, Wenliang, Munn, and Acciaio]{XWMA20}
T.~Xu, L.~K. Wenliang, M.~Munn, and B.~Acciaio.
\newblock {COT-GAN}: Generating sequential data via causal optimal transport.
\newblock \emph{Advances in neural information processing systems},
  33:\penalty0 8798--8809, 2020.

\bibitem[Ye(2011)]{Ye11}
Y.~Ye.
\newblock The simplex and policy-iteration methods are strongly polynomial for
  the {Markov} decision problem with a fixed discount rate.
\newblock \emph{Mathematics of Operations Research}, 36\penalty0 (4):\penalty0
  593--603, 2011.

\bibitem[Yi et~al.(2021)Yi, O'Connor, McGoff, and Nobel]{YOMN21}
B.~Yi, K.~O'Connor, K.~McGoff, and A.~B. Nobel.
\newblock Alignment and comparison of directed networks via transition
  couplings of random walks.
\newblock \emph{arXiv preprint arXiv:2106.07106}, 2021.

\bibitem[Zhang et~al.(2021)Zhang, McAllister, Calandra, Gal, and
  Levine]{zhang2021}
A.~Zhang, R.~McAllister, R.~Calandra, Y.~Gal, and S.~Levine.
\newblock Invariant representations for reinforcement learning without
  reconstruction.
\newblock In \emph{International Conference on Learning Representations
  (ICLR)}, 2021.

\end{thebibliography}

\newpage

\appendix

\section{Extended discussion of related work}\label{app:related_work}

In this appendix we include a discussion of related work that could not be accommodated in the main text due to space limitations.

\subsection{Bisimulation}\label{sec:bisim_discussion}
The concept of bisimulation originated independently in modal logic \citep{vanBenthem1983-VANMLA}, computer science 
\citep{Park81,Milner89} and set theory \citep{FortiH83,Aczel88}, with firm roots in fixed-point theory. Bisimulation 
was originally devised as a tool for determining whether or not two processes are behaviorally equivalent 
in the sense that no test can distinguish between the labels they generate. Being a much less demanding 
notion of equivalence than isomorphism (which is generally NP-hard to verify), the notion of
bisimulation has had significant impact in concurrency theory and formal verification of computer systems, and has 
become a standard tool for model checking. We refer the reader to the very enjoyable paper of 
\citet{sangiorgi2009origins} for a detailed history of bisimulation and related concepts.

\citet{Larsen1989BisimulationTP} developed a theory of probabilistic bisimulation between stochastic processes. Their 
approach is based on comparing interpretations of logical formulas on stochastic processes, roughly saying that two 
processes are probabilistically bisimilar if all formulas acting on the sequence of labels encountered along the 
corresponding random trajectories follow the same probability distribution. Their definition can be most simply 
presented when the two processes live on the same state space and follow the same transition kernel, but are 
initialized at two different states. In this setup, bisimulation reduces to a relation between individual states, which 
can be described formally as follows.
Given a stationary Markov process $\MX=(\mathcal{X}, \PPX, \nuzerox)$ as defined in Section~\ref{sec:prelim}
and a label function $\mathcal{L}:\X\to\real$, a probabilistic bisimulation $R$ is a 
relation on $\X\times\X$ that satisfies the following property: two states $x$ and $x'$ are bisimilar (denoted $xRx'$) 
if and only if $\mathcal{L}(x)=\mathcal{L}(x')$ and for each subset $\mathcal{C}$ in the partition $\X\setminus R$
induced by $R$, it holds that
\[
\sum_{x''\in\mathcal{C}}\PPX(x''|x) = \sum_{x''\in\mathcal{C}}\PPX(x''|x').
\]
\citet{Jonsson91} define another similarity notion called ``satisfaction relation'' based on couplings,
and prove that probabilistic bisimilarity and satisfaction are equivalent notions of similarity.
These works show that bisimulation is indeed an equivalence relation, and that when two processes are initialized from 
two states within the same equivalence class (i.e., they are \emph{bisimilar}), then they will not only transition to 
bisimilar states in the next step but will in fact continue to evolve in a way that is indistinguishable based on the 
labels (in the sense that they will produce the same distribution over sequences of labes).

\citet{Giacalone90}, \citet{DGJP99,DGJP04} and \citet{vBW01} relax the restrictive notion of exact probabilistic 
bisimulation and introduce real-valued pseudometrics that measure the degree of bisimilarity between two states. 
These notions rely on real-valued labeling functions. \citet{Giacalone90} gives a notion of $\varepsilon$-bisimulation 
which relaxes the condition $\mathcal{L}(x) = \mathcal{L}(x')$ in the definition of the hard bisimulation relation 
given above, and only requires equality to hold up to some $\varepsilon > 0$. \citet{DGJP99,DGJP04} go further and 
define a genuinely real-valued extension of bisimulation relations by defining bisimulation metrics as described in the 
main text (and particularly Equation~\eqref{eq:bisimulation_metric}). A fixed-point characterization of these 
bisimulation metrics was established by \citet{vBW01} and \citet{desharnais2002metric}. These works show that the 
resulting distance notion is in fact a pseudometric, and two processes are bisimilar if and only if they are at 
distance zero. This justifies seeing bisimilarity metrics as ``soft'' extensions of the binary relation of bisimilarity.

Interestingly, the first bisimulation metrics all make use of concepts from optimal transport in some way or another: 
\citet{DGJP99,desharnais2002metric} 
already note that their definition is inspired by the Wasserstein distance (which they call the ``Hutchinson 
metric''), and the fixed-point characterization of \citet{vBW01} also make use of this distance to compare transition 
kernels (cf.~Equation~\ref{eq:extended_bellman_v2}). Precisely, their definition that we recalled as 
Equation~\eqref{eq:bisimulation_metric} is admittedly inspired by the Kantorovich dual representation of the 
Wasserstein distance between probability distributions over metric spaces. To our knowledge, this connection 
with optimal transport has not been explored further in the literature, and in particular no ``primal'' counterpart 
based on couplings has been discovered so far.\looseness=-1

\citet{Givan93} adapt probabilistic bisimulation to Markov decision processes (MDPs), requiring states (or state-action 
pairs) to have identical rewards and transition probabilities for two MDPs to be bisimilar.
\citet{FPP04,FCPP06} introduce bisimulation metrics for MDPs, essentially using the fixed-point characterization of 
\citet{vBW01} as a starting point for their definition. Once again, their definition makes use of the Wasserstein 
distance between the transition kernels (called Kantorovich metric in the paper), but no deeper connection between 
bisimulation metrics and optimal transport is discussed.

\cite{castro2020scalable} defines bisimulation metrics for MDPs with respect to a given policy $\pi$, which thus falls 
back to the standard definition of bisimulation metrics for Markov chains as studied in the works of 
\citet{DGJP99,DGJP04} and \citet{vBW01}.
\cite{kemertas22} adjusts the definition of \citet{FPP04,FCPP06} by replacing the Wasserstein distance with the 
entropy-regularized Wasserstein distance proposed by \cite{cuturi2013sinkhorn}, which allows them to apply the Sinkhorn 
algorithm to compute bisimulation metrics for MDPs. The resulting approach can be seen to be nearly identical to the 
methods proposed by \citet{OMN21} and \citet{BWW24} for approximately solving the fixed-point 
equations~\eqref{eq:extended_bellman} in the context of optimal transport---see Section~\ref{sec:OT_discussion} for 
further discussion of these works.

\citet{CvBW12} study the computational complexity of computing bisimulation metrics.
Similar to our work, the authors formulate a linear program that characterizes bisimulation metrics in an equivalent 
way to other common definitions, though the linear program has one constraint per next-state coupling, which makes 
their LP intractable as stated. They use their LP formulation as an analytic tool to show the existence of a 
polynomial-time algorithm to solve the fixed-point equations~\eqref{eq:extended_bellman_v2} (which essentially amounts 
to Bellman's value iteration algorithm implemented in the MDP we describe in Appendix~\ref{app:mdp}). Each step of the 
resulting algorithm solves one optimal-transport problem per state pair via the network simplex algorithm, which is 
known to be impractical for this purpose in comparison with Sinkhorn-style methods \citep{cuturi2013sinkhorn}. In the 
context of optimal transport, a closely related linear program has been discovered by \citet{BBLZ17}, whose framework 
is more general in that it mostly focuses on general (potentially non-Markovian) stochastic processes, but with the 
limitation that only finite horizons are considered. We discuss further developments on this topic in 
Section~\ref{sec:OT_discussion}.

Finally, \citet{BA17} show that $\varepsilon$-bisimilar Markov processes generate distributions over finite-length 
trajectories that are close in total variation distance. Since this distance is a special case of the Wasserstein 
distance (with the Hamming metric over sequences as ground metric), this result can be seen to establish some relation 
between OT distances and bisimulation, but the link is rather weak in the sense that no equivalence is shown between 
the two notions. Indeed, these results only imply that nearly-bisimilar processes generate nearly-identical trajectory 
distributions, but the reverse implication is not shown to hold.

\subsection{Optimal Transport}\label{sec:OT_discussion}
Optimal transport \citep{villani2009optimal} studies the problem of transporting mass between two density functions $p$ 
and $q$, given a cost function that measures the transport distance between any pair of points.
In the classic formulation of \citet{Kan42}, the vehicle used to transport mass is a 
coupling, that is, a joint probability distribution whose marginals equal $p$ and $q$. The problem of finding an optimal 
coupling that minimizes the total transport distance can be formulated as a 
linear program. Historically, the resulting LPs have been solved via standard solvers like the network simplex method 
or interior-point methods, which lead to algorithms with polynomial runtime guarantees but rather poor empirical 
performance. \citet{cuturi2013sinkhorn} successfully advocated for adding entropy regularization to the standard LP 
objective, which enabled algorithms that are orders of magnitude faster than previously proposed methods.

In this work we consider a problem of optimal transport between stochastic processes with a temporal dimension. This 
topic has recently started to receive attention in the OT literature, mostly focusing on stochastic processes with 
finite horizon \citep{pflug12,BBLZ17,lassalle18}. In this setting (often called ``adapted transport'', ``causal 
transport'', or ``bicausal transport''), the problem is to transport mass between joint distributions of sequences of 
elements, which can be formulated as an optimization problem over the set of causal couplings (i.e., the set of 
couplings over joint distributions over sequences that respect the temporal order inherent in the process). 
\citet{BBLZ17} have observed that, due to the linearity of the causality constraints, this optimization problem can be 
phrased as a linear program, which however is infinite-dimensional and thus intractable to solve directly. They 
complement this view by providing dynamic-programming principles for characterizing the structure of the optimal 
coupling, which, in the special case of Markov processes, boils down to the finite-horizon version of the fixed point 
equations~\eqref{eq:extended_bellman}. This development essentially mirrors the LP formulation and dynamic-programming 
principles put forth by \citet{CvBW12} in the context of computing bisimulation metrics (cf.~the discussion in 
Section~\ref{sec:bisim_discussion}).

Still on the front of computing optimal transport distances, a notable contribution is due to \citet{EP24}, who propose 
and analyze a version of Sinkhorn's algorithm for optimal transport on the space of stochastic processes. When 
specialized to Markov processes, their algorithm can be seen to be very closely related to ours, the technical 
explanation being that in finite-horizon Markov processes the entropy of path distributions that they use as 
regularization can be seen to be equal to the conditional entropy that our method uses for the same purpose. The 
resulting algorithm performs iterative Bregman projections via backward recursion over the finite time horizon, with 
computational steps that are essentially identical to applying our Bellman--Sinkhorn operators. That said, their 
analysis relies very heavily on the finite-horizon structure of the problem and as such it is not applicable in our 
considerably more challenging infinite-horizon problem setting.

The more recent works of \cite{moulos2021bicausal,OMN21,BH23} and \citet{BWW24} have investigated optimal-transport 
distances between infinite-horizon Markov chains. \citet{OMN21} considered the undiscounted version of our problem and 
proposed to compute optimal transition couplings via an adaptation of approximate policy iteration (cf.~\citealt{Sch13}) 
to an appropriately adjusted version of the MDP we describe in Appendix~\ref{app:mdp}. Their key algorithmic idea is 
approximating the greedy policy update steps by running Sinkhorn's algorithm for each pair of states. Essentially the 
same idea was used by \citet{BWW24} to solve the discounted problem that is the subject of the present paper, with the 
difference that their method takes approximate value iteration as its starting point. Both of these approaches are 
closely related to the alternative fixed-point definition of bisimulation metrics using Sinkhorn divergences 
due to \citet{kemertas22}, as mentioned in Section~\ref{sec:bisim_discussion}. While these approaches are nearly 
as effective as our Sinkhorn Value Iteration method in practice, their black-box use of Sinkhorn's algorithm make them 
difficult to analyze theoretically, and difficult to build further theory on.

To wrap up, let us mention some results that in a sense have already foreshadowed our observation about the relation of 
OT distances and bisimulation metrics. First, we note that \citet{YOMN21,BWW24} proposed to study optimal transport 
distances of Markov chains defined over graphs as a means of studying the similarity of the underlying graphs. 
The purpose of these works was to define a notion of  distance that is less demanding than isomorphism, but is 
still grounded in fundamental theory and can be computed effectively---which is precisely the reason that the 
notion of bisimulation was originally introduced in the 1980s in the context of formal verification by \citet{Park81} 
and \citet{Milner89}. Finally, the work of \citet{BBBE20} has established that ``all adapted topologies are equal'' on 
the space of laws of stochastic processes, understood in the sense that a large number of topologies (including the one 
induced by optimal-transport metrics defined in terms of bicausal couplings) are in fact identical. While one may argue 
with their sweeping claim that \emph{all} such topologies are equal, it may not be surprising in light of their 
results that the topology induced by bisimulation metrics is also identical to these well-studied topologies (as 
revealed to be true by our observations in this paper).

\newpage
\section{Optimal Transport as a Markov Decision Process}\label{app:mdp}
Consider the two Markov processes $M_\X$ and $M_\Y$ with initial states $x_0$ and $y_0$, respectively. To compute the 
optimal transport cost $\mathbb{W}_\gamma(\MX,\MY;c,x_0,y_0)$ and the optimal transition coupling $\pi^*$, we can 
introduce a Markov decision process 

$\M = (\X\Y, \A, q, \gamma, c, \nu_0)$ 
where:
\begin{itemize}
	\item $\X\Y$ is the state space, defined as the set of joint states of $M_\X$ and $M_\Y$,
	\item $\A(xy)=\Pi_{xy}$ is the set of applicable actions in state $xy\in\X\Y$, corresponding to the set of valid 
couplings of $\PPX(\cdot|x)$ and $\PPY(\cdot|y)$,
	\item $ q(\cdot |xy,a) = a $ is the transition probability distribution, which is fully determined by the action 
$a\in\Pi_{xy}$,  
	\item $ \gamma $ is the discount factor,
	\item $ c: \X\Y \rightarrow [0,\infty) $ is the cost function that maps joint states to positive real numbers,
	\item $ \nu_0 = \delta_{x_0y_0} $ is the initial state distribution.
\end{itemize}
The objective of the agent in this MDP is to select its sequence of actions $A_0,A_1,\dots$ in a way that minimizes 
the total discounted cost $\EE{\sum_{t=0}^\infty \gamma^t c(X_t,Y_t)}$, where each state pair is drawn according to the action 
taken by the agent as $(X_t,Y_t)\sim A_t$. The sequence of actions is generated by a sequence of history dependent
policies $ \pi_t\in\PiHD: \mathcal{H}_t \rightarrow \Delta_{\A(X_t,Y_t)} $ where $\mathcal{H}_t = 
\pa{X_0,Y_0,\dots,X_t,Y_t}$. A first remark is that we can restrict ourselves to deterministic
policies. Indeed, the action set is convex at every time step and both the reward and the transition probability
distributions are linear in the action.
A second remark is that bicausal couplings correspond exactly to history-dependent deterministic policies. Indeed, by
bicausality, the bicausal coupling $ M_{\X\Y} $ is generated by the policy
$ M_{\X\Y}(x_n,y_n|\bar{x}_{n-1}\bar{y}_{n-1}) $. We will denote this policy by $ \pi_{M_{\X\Y}} $.

Of special interest are stationary deterministic (or Markovian) policies of the form $\pi:\X\Y\ra\A$, 
mapping joint states $(X_t,Y_t)$ to actions in $\A(X_t,Y_t)$ as $A_t = \pi(X_t,Y_t)$. Such policies correspond exactly  
 with transition couplings as defined in the main text as mappings $\pi: \X\Y\ra \Delta_{\X\Y}$ of the same 
type. Accordingly, we will sometimes write $\pi(\cdot|xy)$ to refer to the distribution $\pi(xy)\in \A(xy)$ 
below.\looseness=-1

\subsection{Value functions, optimal policies, and sufficiency of transition couplings}\label{app:mdp_values}
Each policy $\pi$ induces a value function $V^\pi$, defined in each state $xy\in\X\Y$ as
\begin{equation}\label{eq:value_function}
V^\pi(xy) = \EEcs{\sum_{t=0}^\infty \gamma^t c(X_tY_t)}{X_0Y_0=xy}{\pi},
\end{equation}
where the expectation is taken with respect to the stochastic process induced by the policy $\pi$.
Here, $X_t$ and $Y_t$ are random variables representing the state of the two processes at time $t$. 
In particular, we have that for any bicausal coupling $ M_{\X\Y}\in \Pibc $,
\begin{align*}
	\int c_{\gamma}(\bar{X},\bar{Y})dM_{\bar{X},\bar{Y}} &= \int \sum_{t=0}^{\infty} \gamma^t c(X_t,Y_t)
	dM_{\X\Y}(\bar{X}, \bar{Y}) \\
							     &= \EEcs{\sum_{t=0}^{\infty} \gamma^t c(X_t,
							     Y_t)}{X_0Y_0 =x_0y_0}{\pi_{M_{\X\Y}}} \\
							     &= V^{\pi_{M_{\X\Y}}}(x_0,y_0).
\end{align*}
And as a result we can relate the optimal transport cost to the optimal value function of the MDP 
\begin{align*}
	\mathbb{W}_\gamma(\MX,\MY;c,x_0,y_0) &= \inf_{\MXY \in \Pibc} \int c_\gamma(\bX, \bY) \,\dd \MXY(\bX,\bY)\\
					     &= \inf_{\MXY \in \Pibc} V^{\pi_{\MXY}}(x_0y_0)\\
					     &= \inf_{\pi \in \PiHD} V^\pi(x_0y_0).
\end{align*}

Building on classic results of MDP theory, it can be shown that there exists an optimal Markovian policy $\pi^*$ whose 
value function $V^* = V^{\pi^*}$ satisfies $V^*(xy) \le V^{\pi}(xy)$ for all policies $\pi$ and joint states $xy$, and 
said optimal value function $V^*$ satisfies the Bellman optimality equations
\[
 V^*(xy) = \TT V^{*}(xy),
\]
where $\TT $ is the Bellman operator acting on a function $ V\in \mathbb{R}^{\X\Y} $ as
\begin{equation*}
	\TT V(xy) = c(xy) + \gamma \inf_{p\in\Pi_{xy}} \sum_{x'y'} p(x'y') V(x'y') \quad \pa{\forall xy}.
\end{equation*}

This is precisely the set of equations in Equation~\eqref{eq:extended_bellman}. Since $V^*$ is optimal, $V^*(xy)$ 
equals the optimal transport cost $\mathbb{W}_\gamma(\MX,\MY;c,x,y)$ for each state $xy$. An optimal policy $\pi^*$ 
achieves the infimum in each state $xy$, with associated value function $V^{\pi^*}=V^*$. These claims are summarized in 
the following theorem, stated in nearly identical form by \citet{moulos2021bicausal}.
\begin{theorem}[cf.~Theorem 1 in \citealt{moulos2021bicausal}]
Under the above conditions, the following hold:
\begin{itemize}
\item There exists an optimal Markovian transition coupling $ \pi^{*} $ such that for any policy $ \pi\in \PiHD $ and 
any joint states $ xy $, $ V^{\pi^{*}}(xy) \leq V^{\pi}(xy)$.
\item The value function of $\pi^*$ satisfies $\mathbb{W}_\gamma(\MX,\MY;c,x,y)=V^{\pi^{*}}(xy)$.
\item There exists a unique solution to the Bellman optimality equation \eqref{eq:extended_bellman} denoted $ V^{*}\in \mathbb{R}^{\X\Y} $.
\item We have that $ V^{*} = V^{\pi^{*}}$.
\end{itemize}
\end{theorem}
The above theorem justifies considering Markovian transition couplings when computing the optimal 
transport cost $\mathbb{W}_\gamma(\MX,\MY;c,x_0,y_0)$.

\begin{proof}
One can straightforwardly check that our setting is an instance of optimal control problems with additive cost
functional studied in \citet{Berts77}. In particular, the contraction assumption (Assumption C in the paper mentioned
above) is satisfied.
 We define pointwise $ V^{*}(xy) = \inf_{\pi\in \PiHD} V^{\pi}(xy)  $,
 and Proposition~1 of \citet{Berts77} shows that $ V^{*} $ is the unique solution to the Bellman optimality equation.

 Now, for the existence of an optimal stationary policy or transition coupling in our terminology, an additional
 technical condition on the action set must be carefully verified.  For every $ xy\in \X\Y $, $\ell \in [0,\infty) $ 
and $k$, we define the set
 \begin{equation*}
 	U_k(xy, \lambda) = \left\{ a \in \mathcal{A}(xy) : c(x,y) + \gamma \int \TT^kV(x'y') \dd a(x'y') \leq \lambda 
\right\}
 \end{equation*}
This set is compact as the intersection of the compact set $ \mathcal{A}(xy) $ with a closed set, the preimage of a closed set
 by a continuous application.
 Knowing that $ U_k(xy, \lambda) $ is compact, we can then apply Proposition 14 of \citet{Berts77} and that gives us the
 existence of an
optimal Markovian transition coupling $ \pi^{*} $ that satisfies $ V^{\pi^{*}} = V^{*} $. In particular, for any policy
$ \pi \in \PiHD $ and joint state $ xy $, we have that $ V^{\pi^{*}}(xy) =V^{*}(xy) \leq  V^{\pi}(xy)$ by definition
of $ V^{*} $.
\end{proof}

\subsection{Occupancy measures and occupancy couplings}\label{app:occupancy_validity}
In a finite Markov decision process, occupancy measures express the discounted number of times that a given state and 
action are visited on expectation by the controlled stochastic process. This notion is not meaningfully applicable in 
the MDP formulation of our optimal-transport problem, given that the action space is infinite. However, the closely 
related notion of occupancy coupling can be seen to play a similar role in that it allows expressing the 
total-discounted-cost objective as a linear function, and that the set of valid occupancy couplings can be fully 
characterized in terms of a finite number of linear constraints. In what follows, we prove this latter key property of 
occupancy couplings, stated as Lemma~\ref{lem:occupancy_validity} in the main text.

\begin{proof}[Proof of Lemma \ref{lem:occupancy_validity}]
We begin by showing that the occupancy coupling $\mu^\pi$ associated with any transition coupling $\pi\in\Pi_{\X\Y}$ 
satisfies the following system of equations (sometimes called the ``Bellman flow equations''):
\begin{equation}\label{eq:flow1}
 \mu^\pi(xy,x'y') = \pi(x'y'|xy)\pa{\gamma \sum_{x''y''} \mu^\pi(x''y'',xy) + (1-\gamma) \nu_0(xy)}.
\end{equation}
Indeed, this can be shown to follow from the definition of occupancy couplings as
\begin{align*}
 \mu^\pi(xy,x'y') &= (1-\gamma) \sum_{t=0}^\infty \gamma^t \PPpi{X_tY_t = xy,X_{t+1}Y_{t+1}=x'y'} 
 \\
 &= (1-\gamma) \sum_{t=0}^\infty \gamma^t \pi(x'y'|xy)\PPpi{X_tY_t = xy} 
 \\
 &= \pi(x'y'|xy)\pa{(1-\gamma) \nu_0(xy) + 
(1-\gamma) \sum_{t=1}^\infty \gamma^t \PPpi{X_tY_t = xy}}
\\
&= \pi(x'y'|xy)\pa{(1-\gamma) \nu_0(xy) + \gamma \sum_{x''y''} (1-\gamma) \sum_{t=1}^\infty \gamma^{t-1} 
\PPpi{X_{t-1}Y_{t-1}=x''y'',X_tY_t = xy}}
\\
&= \pi(x'y'|xy)\pa{(1-\gamma) \nu_0(xy) + \gamma \sum_{x''y''} \mu^\pi(x''y'',xy)},
\end{align*}
where in the first step we used the stationarity of the transition coupling $\pi$, then the definition of $\nu_0$, 
followed by the law of total probability, and finally the stationarity of the Markov chain that allowed us to recognize 
$\mu^\pi(x''y'',xy)$ in the last step. Now, summing both sides of Equation~\eqref{eq:flow1} for all $x'y'$, we can 
confirm that $\mu^\pi$ indeed satisfies Equation~\eqref{eq:const_flow}. Furthermore,
summing the two sides of Equation~\eqref{eq:flow1} over all $ x ' $, we get
\begin{align*}
	\sum_{x'} \mu^\pi(xy,x'y') &= \sum_{x'} \pi(x'y'|xy) \pa{(1-\gamma) \nu_0(xy) + \gamma \sum_{x''y''} \mu^{\pi}(x 
''y 
'', xy)} \\
			       &= \sum_{x'} \pi(x'y'|xy) \sum_{x ''y ''} \mu^\pi(xy, x''y'')  \\
			       &= \PPY(y'|y) \sum_{x ''y ''} \mu^\pi(xy, x''y''),
\end{align*}
where the first step in the second line follows from Equation~\ref{eq:const_flow}, and the 4th line comes from the 
fact that $ \pi(\cdot|xy)$ is a coupling of $\PPX(\cdot|x)$ and $\PPY(\cdot|y)$. This verifies that $\mu^\pi$ satisfies
Equation~\eqref{eq:const_Y}, and the same reasoning can be used to verify that it also 
satisfies Equation~\ref{eq:const_X}.

Conversely, suppose that $ \mu \in\real_+^{\X\Y\times\X\Y}$ satisfies Equations~\eqref{eq:const_flow}, 
\eqref{eq:const_X}, and \eqref{eq:const_Y}. Define $\nu_\mu(xy) = \sum_{x^{\prime}y^{\prime}} 
\mu(xy,x^{\prime}y^{\prime})$ and let
\begin{align*}
	\pi_\mu(x^{\prime}y^{\prime}|xy) =
\begin{cases}
	\frac{\mu(xy,x^{\prime}y^{\prime})}{\nu_\mu(xy)} &\text{ if } \nu_\mu(xy) \neq 0, \\
	\PPX(x^{\prime}|x)\PPY(y^{\prime}|y) &\text{ otherwise } .
\end{cases}
\end{align*}
We will verify that $ \pi_{\mu} $ defines a valid Markovian coupling and that $ \mu $ is the state action occupancy 
measure of $ \pi_{\mu} $. If $ \nu_\mu(xy) = 0 $, then 
$ \sum_{x'} \pi_{\mu}(x'y'|xy) = \PPY(y'|y) \sum_{x'} \PPX(x'|x) = \PPY(y'|y) $. If $ \nu_\mu(xy) \neq 0 $, then we have 
\begin{align*}
	\sum_{x'} \pi_{\mu}(x'y'|xy) &= \frac{1}{\nu_\mu(xy)} \sum_{x'} \mu(xy,x'y') = \frac{1}{\nu_\mu(xy)} \sum_{x'y''} \mu(xy, 
x'y'') \PPY(y '|y) \\
				     &= \frac{1}{\nu_\mu(xy)} \nu_\mu(xy) \PPY(y'|y) = \PPY(y'|y),
\end{align*}
where we have used that $\mu$ satisfies the constraint of Equation~\eqref{eq:const_Y}.
In any case, we have that $ \sum_{x'}\pi_{\mu}(x'y'|xy) = \PPY(y'|y) $. By symmetry, using that $\mu$ satisfies 
Equation~\eqref{eq:const_X}, we also have $ \sum_{y'} \mu (x' y'|xy 
) =\PPX(x'|x)$. Hence we have demonstrated that $ \pi_{\mu} $ is a valid Markovian coupling of $\PPX$ and $\PPY$.
To proceed, observe that the occupancy coupling associated with $\pi_\mu$ satisfies
\[
\sum_{x'y'}\mu^{\pi_\mu}(xy,x'y') = \gamma \sum_{x''y''} \mu^{\pi_\mu}(x''y'',xy) + (1-\gamma) \nu_0(xy).
\]
We will verify that $\mu^{\pi_\mu} = \mu$ by showing that this system of equations has a unique solution. In order to 
see this, let us recall the definition of $\nu_\mu$ and reorder Equation~\eqref{eq:const_flow} as
\[
 (1-\gamma)\nu_0(xy) = \nu_\mu(xy) - \gamma \sum_{x''y''} \nu_\mu(x''y'') \pi_\mu(xy|x''y'').
\]
Introducing the matrix 
$Z\in\real^{|\X||\Y|\times|\X||\Y|}$ with entries $Z(xy,x'y') = \pi_\mu(x'y'|xy)$ and representing the functions 
$\nu_0$ and $\nu_\mu$ in matrix form, this system of equations can be written as
\[
 (1-\gamma) \nu_0 = (I - \gamma Z) \nu_{\mu}.
\]
Now, thanks to the Perron--Frobenius theorem, the stochastic matrix $Z$ has spectral radius $1$, and thus  $(I - \gamma 
Z)$ is 
invertible, meaning that there is a unique solution $\nu_\mu = (1-\gamma)(I - \gamma Z)^{-1} \nu_0$. This in 
turn implies that $\mu = \mu^{\pi_\mu}$, thus verifying that $\mu$ is indeed an occupancy coupling induced by a valid 
transition coupling $\pi_\mu$ if it satisfies Equations~\eqref{eq:const_flow}-\eqref{eq:const_Y}. This concludes the 
proof.
\end{proof}

\newpage
\section{The proof of Theorem~\ref{thm:main}}\label{app:main}
The proof is composed of two main parts:  showing a bound on the \emph{regret} of the iterates 
$\mu_1,\mu_2,\dots,\mu_K$, and then 
accounting for the errors incurred when rounding the average iterate $\bmu_K = \frac 1K \sum_{k=1}^K \mu_k$ to a 
feasible occupancy coupling. We start by stating 
some general results that will be useful throughout the analysis, and then study the two sources of error mentioned 
above separately.
For any $ \mu \in \mathbb{R}^{\X\Y\times\X\Y} $, we define $ E \mu (xy) = \sum_{x',y'} \mu(x'y', xy)$, $ E\transpose 
\mu(xy) = \sum_{x',y'} \mu(xy, x'y') $ and 
\begin{align*}
	\pi_\mu(x^{\prime}y^{\prime}|xy) =
\begin{cases}
	\frac{\mu(xy,x^{\prime}y^{\prime})}{E\transpose\mu(xy)} \text{ if } E\transpose\mu(xy) \neq 0, \\
	\PPX(x^{\prime}|x)\PPY(y^{\prime}|y) \text{ otherwise} .
\end{cases}
\end{align*}
In particular, we always have that $ \mu(xy,x'y') = E\transpose \mu(xy) \pi_{\mu}(x' y'|xy) $.
Furthermore, for any $ \pi:\X\Y\ra\Delta_{\X\Y}$, we will denote by $ \tilde{\pi} = \rho(\pi)$ the 
rounded transition coupling obtained by adapting the rounding procedure of \citet{altschuler2017near}, and we will let 
$\rho(\mu)$ 
denote the corresponding occupancy coupling $\mu^{\tilde{\pi}_\mu}$ (cf.~Section~\ref{app:rounding} for the description 
and analysis of this method). In particular, $ \tilde{\pi} $ will always be a valid transition coupling associated with 
$ \PPX, \PPY$. 
Finally we define $ \piout = \rho\bpa{\pi_{\bmu_K}} $ and $ \muout = \mu^{\piout} $,
and we let $\mu^*$ be an optimal occupancy coupling achieving the minimum in the problem 
formulation of Theorem~\ref{thm:LP}. 

With these notations, we decompose the overall error of the output as
\[
 \iprod{\muout - \mu^*}{c} = \iprod{\muout - \bmu_K}{c} + \iprod{\bmu_K - \mu^*}{c} = 
\iprod{\muout - \bmu_K}{c} + \frac 1K \sum_{k=1}^K \iprod{\mu_k - \mu^*}{c}.
\]
The first sum on the right-hand side corresponds to the rounding error, and the second one to the so-called 
\emph{regret} of the sequence of iterates $\mu_k$. This latter sum can be controlled by adapting arguments from the 
classic analysis of mirror-descent methods \citep{NY83,BT03}, with some ideas adopted from the Mirror Sinkhorn analysis 
of \citet{BB23} that will also come in handy for analyzing the rounding errors. We state these tools first below, and then 
analyze the two terms in the above decomposition separately.

\subsection{General tools}
We begin with a version of the classic ``three-point identity'' for mirror-descent methods (e.g., Lemma~4.1 of 
\citealp{BT03}) adapted to our specific setting that involves alternating projections to the sets $\BX$ and $\BY$. The 
result is similar to Lemma~A.3 of \citet{BB23}, which we reprove here with a more standard methodology (as used for 
proving, e.g., Theorem 28.4 of \citealp{lattimore2020bandit}).
\begin{lemma}\label{lem:MDregret}
Let $\mu^*\in\BX\cap\BY$ be arbitrary. Then,
  \[
 \iprod{\mu_{k+1} - \mu^*}{c} \le \frac{\HH{\mu^*}{\mu_k} - \HH{\mu^*}{\mu_{k+1}} - \HH{\mu_{k+1}}{\mu_k}}{\eta}. 
\]
\end{lemma}
\begin{proof}
We first recall that $\mathcal{H}$ is the Bregman divergence associated with the conditional entropy function $\CC(\mu) 
= \sum_{xy,x'y'} \mu(xy,x'y') \log \frac{\mu(xy,x'y')}{\sum_{x''y''} \mu(xy,x''y'')}$ (cf.~Appendix~A.1 in 
\citealt{NJG17}). For the actual proof, 
let us consider the case of $k$ odd when $\mu_{k+1}\in\BX$. Then, by definition, we have that $\mu_{k+1}$ is the 
minimizer of $\Psi_{k+1}(\mu) = \eta \iprod{\mu}{c} + \HH{\mu}{\mu_k}$ on this set. Noticing that the gradient of 
$\Psi_{k+1}$ at $\mu$ is written as $\nabla \Psi_{k+1}(\mu) = \eta c + \nabla \CC(\mu) - \nabla 
\CC(\mu_k)$, the 
first-order optimality condition over the convex set $\BX$
implies that the following inequality holds for any $\mu\in\BX$:
\[
 \iprod{\eta c + \nabla \CC(\mu_{k+1}) - \nabla \CC(\mu_k)}{\mu - \mu_{k+1}} \ge 0.
\]
In particular, using this result for $\mu = \mu^*$ (which is indeed in $\BX$), the claim follows from using the 
standard three-point identity of Bregman divergences that states
\[
\iprod{\nabla \CC(\mu_{k+1}) - \nabla \CC(\mu_k)}{\mu^* - \mu_{k+1}} = \HH{\mu^*}{\mu_{k}} - 
\HH{\mu^*}{\mu_{k+1}} - \HH{\mu_{k+1}}{\mu_k}.
\]
Repeating the same argument for even rounds (and noticing that the comparator also satisfies $\mu^*\in\BY$ as 
needed for that case) completes the proof.
\end{proof}
The following standard lemma will also be useful for studying various notions of distances between occupancy couplings: 
the total variation distance $\onenorm{\mu - \mu'} = \sum_{xy,x'y'} |\mu(xy,x'y') - \mu'(xy,x'y')|$, the relative 
entropy $\DD{\mu}{\mu'} = \sum_{xy,x'y'} \mu(xy,x'y') \log \frac{\mu(xy,x'y')}{\mu'(xy,x'y')}$, and the conditional 
relative entropy introduced earlier. 
\begin{lemma}\label{lem:H_vs_D}
For any two occupancy couplings $\mu$ and $\mu'$, we have
\[
\frac 12 \onenorm{\mu-\mu'}^2 \le \DD{\mu}{\mu'} \le \frac{\HH{\mu}{\mu'}}{1-\gamma}.
\]
\end{lemma}
\begin{proof}
The first inequality is Pinsker's. For proving the second inequality, we let $\nu = E^T\mu$ and $\nu' = E^T \mu'$ be the state occupancies 
associated with $\mu$ and $\mu'$, respectively.
Then, we write the following:
\begin{align*}
 \DD{\mu}{\mu'} &= \DD{\nu}{\nu'} + \HH{\mu}{\mu'}
 \\
 &\qquad\qquad\mbox{(by the chain rule of the relative entropy)}
 \\
 &= \DD{(1-\gamma) \nu_0 + \gamma E\mu}{(1-\gamma) \nu_0 + \gamma E\mu'} 
 + \HH{\mu}{\mu'}
\\
 &\qquad\qquad\mbox{(using that $\mu$ and $\mu'$ are valid occupancy couplings)}
\\
&\le (1-\gamma) \DD{\nu_0 }{\nu_0} + \gamma \DD{E\mu}{E\mu'}
+ \HH{\mu}{\mu'}
\\
 &\qquad\qquad\mbox{(using the joint convexity of the relative entropy)}
\\
&\le \gamma \DD{\mu}{\mu'} + \HH{\mu}{\mu'},
\end{align*}
where the final step follows from using the data-processing inequality for the relative entropy. Reordering the 
terms concludes the proof.
\end{proof}

\subsection{Constraint violations}
Let us begin with some definitions. First of all we introduce a quantity that measures the extent to which 
an occupancy coupling $\mu$ violates the transition coherence constraints. Specifically, we will measure the violations 
of the $\X$-constraints by
\[
\delta_\X(\mu) = \sum_{xyx^{\prime}}\left\lvert\nu_\mu(xy)\PPX(x^{\prime}|x) - 
\sum_{y^{\prime}}\mu(xy,x^{\prime}y^{\prime})\right\rvert
\]
and the violations of the $\Y$-constraints by
\[
\delta_\Y(\mu) = \sum_{xyx^{\prime}}\left\lvert\nu_\mu(xy)\PPY(y^{\prime}|y) - 
\sum_{x^{\prime}}\mu(xy,x^{\prime}y^{\prime})\right\rvert,
\]
and the overall constraint violations will be written as
\[
\delta(\mu) = \delta_\X(\mu) + \delta_{\Y}(\mu).
\]

Note that we have $\delta_\X(\mu_k) = 0$ for odd rounds and $\delta_\Y(\mu_k)=0$ for even % 
rounds by definition of the updates. 
We also define the rounding error associated with an occupancy coupling $\mu$ as the average total variation distance 
between the transition coupling $\pi_\mu$ and its rounded counterpart $\wt{\pi}_\mu = \rho(\pi_\mu)$:
\[
 \Delta(\mu) = \sum_{xy,x'y'}\mu(xy,x'y')\onenorm{\wt{\pi}_{\mu}(\cdot|xy) - \pi_{\mu}(\cdot|xy))}.
\]
The first statement establishes a link between the quality of an occupancy coupling and the occupancy coupling obtained 
after rounding the transition coupling to satisfy the constraints. 
\begin{lemma}
    \label{lemma:rounding}
    For any $\mu$ satisfying Equation~\eqref{eq:const_flow}, we have
    \[
    \iprod{\rho(\mu) - \mu}{c} \leq \infnorm{c} \frac{\Delta(\mu)}{(1-\gamma)}.
 \]
\end{lemma}
\begin{proof}
	Let $\tilde{\pi}_\mu = \rho(\pi_\mu)$, $\tilde{\mu} = \rho(\mu)$, $ \nu_{\mu} =E\transpose \mu $, and $ \nu_{\tmu} 
= E^T \tmu$. Furthermore, let us define the shorthand notation $\nu\circ\pi$ to denote the composition of the 
state-pair distribution $\nu$ with the transition coupling $\pi$ as $(\nu\circ\pi)(xy,x'y') = \pi(x'y'|xy) \nu(xy)$. 
Then, we have
 \begin{align*}
 \onenorm{\tmu- 
\mu} &= \onenorm{\nu_{\tmu}\circ\tpi_{\mu} - \nu_{\mu}\circ \pi_{\mu}} = \onenorm{\nu_{\tmu} \circ \tpi_{\mu} + 
\nu_{\mu}\circ \tpi_{\mu} - \nu_{\mu}\circ\tpi_{\mu} -\nu_{\mu}\circ\pi_{\mu}} 
  \\
 &
 = \onenorm{(\nu_{\tmu}-\nu_{\mu})\circ\tpi_{\mu} + \nu_{\mu}\circ(\tpi_{\mu} - \pi_{\mu})}
   \\
   & \le \onenorm{(\nu_{\tmu}-\nu_{\mu})\circ\tpi_{\mu}} + \onenorm{\nu_{\mu}\circ(\wt{\pi}_{\mu} - \pi_{\mu})}
\\
 &\qquad\qquad\mbox{(using the triangle inequality)}
   \\
 &
 = \onenorm{\nu_{\tmu}-\nu_{\mu}} + \Delta(\mu)
\\
 &\qquad\qquad\mbox{(using the definition of $\Delta$)}
 \\
 &= \gamma \onenorm{E \tmu- E \mu} + \Delta(\mu)
 \\
 &\qquad\qquad\mbox{(using $\nu_\mu = (1-\gamma)\nu_0 + \gamma E \mu$)}
 \\
  &
 \le \gamma \onenorm{\tmu- \mu} + \Delta(\mu),
\end{align*}
where in the last step we used that $E$ is non-expansive with respect to the $\ell_1$-norm.
Reordering gives
\[
 \onenorm{\tmu- \mu} \le \frac{\Delta(\mu)}{1-\gamma},
\]
and putting everything together proves the claim of the lemma.
\end{proof}

The next result relates the rounding errors to the constraint violations.
\begin{lemma} 
\label{lemma:Delta_to_delta}
For any $ \mu \in \Delta_{\X\Y\times\X\Y}$, we have
    $$\frac{1}{2}\Delta(\mu)\leq \delta(\mu)$$
\end{lemma}
The proof of this result builds on the error analysis of the rounding procedure of \citet{altschuler2017near}, and can 
be found in Appendix~\ref{app:Delta_to_delta}.
Finally, the last technical lemma (inspired by Lemma~A.4 of \citealp{BB23}) bounds the rounding errors in 
terms of the change rate of the occupancy couplings.
\begin{lemma}
\label{lem:delta_to_mu}
For any $k\geq 1$,
    $$\delta(\mu_k)\leq 2\min(\onenorm{\mu_k - \mu_{k+1}}, \onenorm{\mu_k - \mu_{k-1}}).$$
\end{lemma}
\begin{proof}
We study the case where the $k$th update is a $\BY$ projection. For this proof, it will be convenient to introduce the 
following notation. We define $\IX:\real^{\X\Y\times\X\Y}\ra \real^{\X\Y\times\X}$ and (with some abuse of 
notation), $\PPX: \real^{\X\Y\times\X\Y}\ra \real^{\X\Y\times\X}$ as the linear operators that respectively act on 
$\mu$ via the assignment $ (\IX \mu)(xy,x') = \sum_{y'}\mu(xy,x'y')$ and  $ (\PPX \mu)(xy,x') = 
\sum_{x'',y''}\mu(xy,x''y'') \PPX(x'|x)$. This allows us to write $\delta_{\X}(\mu) = \onenorm{(\IX - \PX)\mu}$, so 
that we have the expression
\[
 \delta(\mu_k) = \delta_{\X}(\mu_k) + \delta_{\Y}(\mu_k) = \onenorm{\IX \mu_k - \PPX \mu_k},
\]
where we have also used $\delta_{\Y}(\mu_k)=0$ that holds thanks to the fact that $k$ is even.
Moreover, the $(k+1)$st update is a $ \BX $-projection and thus we have $ \IX
\mu_{k+1} = \PPX \mu_{k+1} $. Hence,
\begin{align*}
	\delta(\mu_k) &= \onenorm{\IX \mu_k - \PPX \mu_k} \\
		      &= \onenorm{\IX \mu_k -\IX \mu_{k+1} + \PPX \mu_{k+1} - \PPX \mu_k}\\
		      &\leq \onenorm{\IX (\mu_k - \mu_{k+1})} + \onenorm{\PPX (\mu_{k+1} - \mu_k)}\\
		      &\leq 2 \onenorm{\mu_k - \mu_{k+1}},
\end{align*}
where we have used the fact that both $ \IX $ and $ \PPX $ are non-expansions for the $\ell_1$-norm in the last line, 
which follows from the data-processing inequality for the total variation distance. We then conclude the 
analysis for the even rounds by replacing $\mu_{k+1}$ with $\mu_{k-1}$ in the argument above, and repeating the same 
reasoning for odd rounds completes the overall proof.
\end{proof}

\allowdisplaybreaks[3]
Having established these elementary results, we now turn to addressing the main technical hurdle: bounding the 
cumulative rounding errors.
\begin{theorem} \label{thm:cumulative_delta}
The cumulative constraint violations of the iterates produced by Sinkhorn Value Iteration satisfy
 \[
  \sum_{k=1}^K \delta(\mu_k) \le \frac{(1-\gamma)\HH{\mu^*}{\mu_1}}{2\eta \infnorm{c}} + \frac{16\eta \infnorm{c}
  K}{(1-\gamma)^2}.
 \]
\end{theorem}
\begin{proof}
We start by applying Lemma~\ref{lem:delta_to_mu} to show $\delta(\mu_k) \le 2\onenorm{\mu_k - \mu_{k+1}}$, which reduces 
our task to bounding $\sum_{k=1}^K \onenorm{\mu_k - \mu_{k+1}}$. We do this as follows, for any fixed $\alpha > 0$:
\begin{align*}
 \sum_{k=1}^K \onenorm{\mu_k - \mu_{k+1}} &\le \frac{\alpha 
K}{2} + \sum_{k=1}^K \frac{\onenorm{\mu_k - \mu_{k+1}}^2}{2 \alpha}
 \\
 &\qquad\qquad\mbox{(by the inequality of arithmetic and geometric means)}
 \\
 & \le \frac{\alpha K}{2} + \sum_{k=1}^K \frac{\DD{\mu_{k+1}}{\mu_{k}}}{\alpha}
 \\
 &\qquad\qquad\mbox{(Pinsker's inequality)}
 \\
 & \le \frac{\alpha K}{2} + \sum_{k=1}^K \frac{\HH{\mu_{k+1}}{\mu_{k}}}{(1-\gamma)\alpha}
 \\
 &\qquad\qquad\mbox{(Lemma~\ref{lem:H_vs_D})}
 \\
 & \le \frac{\alpha K}{2} + \sum_{k=1}^K \frac{\HH{\mu^*}{\mu_k} - \HH{\mu^*}{\mu_{k+1}}}{(1-\gamma) \alpha} + 
\frac{\eta}{\alpha\pa{1-\gamma}}\sum_{k=1}^K \iprod{c}{\mu^* - \mu_{k+1}}
 \\
 &\qquad\qquad\mbox{(Lemma~\ref{lem:MDregret})}
 \\
 & \le \frac{\alpha K}{2} + \frac{\HH{\mu^*}{\mu_1} - \HH{\mu^*}{\mu_{K+1}}}{\alpha (1-\gamma) } + 
\frac{\eta}{\alpha\pa{1-\gamma}}\sum_{k=1}^K \iprod{c}{\rho(\mu_{k+1}) - \mu_{k+1}}
 \\
 &\qquad\qquad\mbox{(since $\iprod{\mu^*}{c} \le \iprod{\rho(\mu_{k+1})}{c}$)}
 \\
 & \le \frac{\alpha  K}{2} + \frac{\HH{\mu^*}{\mu_1}}{\alpha (1-\gamma)} + \frac{\eta 
\infnorm{c}}{\alpha (1-\gamma)^2}\sum_{k=1}^K \Delta(\mu_{k+1})
 \\&\qquad\qquad\mbox{(Lemma~\ref{lemma:rounding})}\\
 & \le \frac{\alpha  K}{2} + \frac{\HH{\mu^*}{\mu_1}}{\alpha (1-\gamma)} + \frac{2\eta 
\infnorm{c}}{\alpha (1-\gamma)^2}\sum_{k=1}^K \delta(\mu_{k+1})
 \\&\qquad\qquad\mbox{(Lemma~\ref{lemma:Delta_to_delta})}\\
 & \le \frac{\alpha  K}{2} + \frac{\HH{\mu^*}{\mu_1}}{\alpha (1-\gamma)} + \frac{4\eta 
\infnorm{c}}{\alpha (1-\gamma)^2}\sum_{k=1}^K \onenorm{\mu_k - \mu_{k+1}},
\end{align*}
where we have finally used that $\delta(\mu_{k+1}) \le 2\onenorm{\mu_{k}-\mu_{k+1}}$ (Lemma \ref{lem:delta_to_mu}). Now, 
we need to make sure that $\frac{4\eta 
\infnorm{c}}{\alpha (1-\gamma)^2} \le 1$ to turn this into a meaningful result. In particular, setting $\alpha = 
\frac{8\eta \infnorm{c}}{(1-\gamma)^2}$ guarantees that the constant in question equals $\frac 12$, so that we can reorder 
the terms to obtain
\begin{align*}
 \sum_{k=1}^K \onenorm{\mu_k - \mu_{k+1}} &\le \alpha  K + \frac{2\HH{\mu^*}{\mu_1}}{\alpha (1-\gamma)} = 
 \frac{8\eta \infnorm{c}K}{(1-\gamma)^2} + \frac{(1-\gamma)\HH{\mu^*}{\mu_1}}{4\eta \infnorm{c}}.
\end{align*}
This concludes the proof.
\end{proof}

\subsection{Regret analysis}
In this section, we bound the regret of the iterates produced by Sinkhorn Value Iteration.
\begin{theorem}\label{thm:regret}
The regret of the mirror Sinkhorn procedure satisfies
 \[
  \sum_{k=1}^K \iprod{\mu_k - \mu^*}{c} \le \frac{\HH{\mu^*}{\mu_1}}{\eta} + 2\infnorm{c}.
 \]
\end{theorem}
\begin{proof}
We first apply Lemma~\ref{lem:MDregret} to obtain the bound
\begin{equation}\label{eq:regret_loose}
 \iprod{\mu_k - \mu^*}{c} \le \frac{\HH{\mu}{\mu_k} - \HH{\mu}{\mu_{k+1}} - \HH{\mu_{k+1}}{\mu_k}}{\eta} + 
\iprod{c}{\mu_{k+1} - \mu_k}.
\end{equation}
Adding up both sides for all $k=1,2,\dots,K$, we get
\begin{align*}
 \sum_{k=1}^K \iprod{\mu_k - \mu^*}{c} &\le \frac{\HH{\mu}{\mu_1} - \HH{\mu}{\mu_{K+1}} - 
\sum_{k=1}^K \HH{\mu_{k+1}}{\mu_k}}{\eta} + \iprod{c}{\mu_{K+1} - \mu_1}
\\
&\le\frac{\HH{\mu}{\mu_1}}{\eta} + 2\infnorm{c}.
\end{align*}
This concludes the proof.
\end{proof}

\subsection{The proof of Theorem~\ref{thm:main}}\label{app:proof_final_steps}
The proof follows from applying the above results to bounding the rounding errors and the regret. The first of these is 
handled as follows:
\begin{align*}
	\iprod{\muout - \bmu_K}{c} & \leq \infnorm{c} \frac{\Delta(\bmu_K)}{1- \gamma} 
				    \leq 2 \infnorm{c} \frac{\delta(\bmu_K)}{1 - \gamma} 
				    \\
				    &\leq 2 \infnorm{c} \frac{\sum_{k=1}^{K}\delta(\mu_k)}{K(1 - \gamma)} 
				    \leq \frac{\HH{\mu^{*}}{\mu_1}}{K \eta} + \frac{32\eta \infnorm{c}^2}{(1-
				   \gamma)^3},
\end{align*}
where the first and second inequalities respectively come from Lemmas~\ref{lemma:rounding} 
and~\ref{lemma:Delta_to_delta}, the third one from the convexity of $\delta$, and the last inequality comes from 
the application of Theorem~\ref{thm:cumulative_delta}.

The regret is then bounded using Theorem~\ref{thm:regret} as
\begin{align*}
\iprod{\bmu_K - \mu^*}{c} & = \frac{1}{K} \sum_{k=1}^K\iprod{{\mu}_k - \mu^*}{c}  
  \le \frac{\HH{\mu^*}{\mu_1}}{K\eta} + \frac{2}{K}\infnorm{c}.
\end{align*}

Putting both bounds together, we obtain
\begin{align*}
	\iprod{\muout - \mu^{*}}{c} &= \iprod{\muout - \bmu_K}{c} + \iprod{\bmu_K - \mu^{*}}{c} \\
				    &\leq \frac{2 \HH{\mu^{*}}{\mu_1}}{K \eta} + \frac{32 \eta
				    \infnorm{c}^{2}}{(1-\gamma)^3} + \frac{2 \infnorm{c}}{K}\\
\end{align*}

Now, to prove the actual claim of the theorem, we now let $\pi_1$ be the uniform transition coupling and $\mu_1$ be the 
associated occupancy coupling. In this case, the conditional relative entropy can be upper bounded as 
$\HH{\mu^{*}}{\mu_1} \le \log |\X||\Y|$, and we can further bound 
\begin{align*}
	\iprod{\muout - \mu^{*}}{c} &\le \frac{2\log |\X||\Y|}{K \eta} + \frac{32 \eta \infnorm{c}^2}{(1-\gamma)^3} +
				    \frac{2 \infnorm{c}}{K}\\
				    &\leq 16\infnorm{c} \sqrt{\frac{\log|\X||\Y| }{K(1-\gamma)^3}} + \frac{2
				    \infnorm{c}}{K}\\
				    &\leq 18\infnorm{c} \sqrt{\frac{\log |\X||\Y|}{K(1-\gamma)^3}},
\end{align*}
where the  second to last line is obtained by picking $ \eta = \frac 1 {4\infnorm{c}} \sqrt{\frac{(1-\gamma)^3 \log |\X| 
|\Y|}{K}}$ and the last line is obtained by noticing that $ \frac{1}{K} \leq 
\sqrt{\frac{\log{|\X||\Y|}}{K(1-\gamma)^3}} $ holds whenever $K\geq 1$, $|\X|\geq 2$, $|\Y|\geq 2$ and 
$\gamma\in(0,1)$.
Finally, note that 
\begin{align*}
V^\piout(x_0y_0) - \mathbb{W}_\gamma(\MX,\MY;c,x_0,y_0) &= \frac 1 {1-\gamma} \iprod{\muout - \mu^{*}}{c} \leq 
18\infnorm{c} \sqrt{\frac{\log |\X||\Y|}{K(1-\gamma)^5}}.
\end{align*}
Now, it can be directly verified that the right-hand side is indeed at most $\varepsilon$ whenever $K$ is greater than 
the expression given in the theorem, thus concluding the proof.
\qed

\newpage
\section{Sinkhorn Policy Iteration}\label{app:SPI}
We describe here a simple alternative to Sinkhorn Value Iteration called Sinkhorn Policy Iteration (SPI). 
After introducing this method heuristically, we provide a formal performance analysis, and finally explain its 
relation to SVI.

The core concept underlying the definition of SPI is the notion of \emph{Q-functions}, defined analogously to 
action-value functions in an MDP. The Q-function associated with a transition coupling $\pi$ is a function 
$Q^{\pi}:\X\Y\times\X\Y\ra\real$, with each of its entries defined as
\[
 Q^\pi(xy,x'y') = \EEcs{\sum_{t=0}^\infty \gamma^t c(X_t,Y_t)}{(X_0,Y_0) = (x,y), (X_1,Y_1) = (x'y')}.
\]
Analogously to the results presented in Section~\ref{app:mdp_values}, it is possible to show that the Q-function of a 
given transition coupling $\pi$ satisfies the Bellman equations
\[
 Q^\pi(xy,x'y') = c(xy) + \gamma \sum_{x''y''} \pi(x''y''|x'y') Q^\pi(x'y',x'',y''),
\]
and that the Q-value function of the optimal transition coupling $Q^* = Q^{\pi^*}$ satisfies the Bellman optimality 
equations
\[
 Q^*(xy,x'y') = c(xy) + \gamma \inf_{p\in\Pi_{x'y'}} \sum_{x''y''} p(x''y'') Q^*(x'y',x'',y'').
\]
These are respectively the fixed points of the Bellman operator 
$\TT^\pi:\real^{\X\Y\times\X\Y}\ra\real^{\X\Y\times\X\Y}$ defined via
\[
 (\TT^\pi f)(xy,x'y') = c(xy) + \gamma \sum_{x''y''} \pi(x''y''|x'y') f(x'y',x'',y'')
\]
and the Bellman optimality operator $\TT:\real^{\X\Y\times\X\Y}\ra\real^{\X\Y\times\X\Y}$ defined via
\[
 (\TT f)(xy,x'y') = c(xy) + \gamma \inf_{p\in\Pi_{x'y'}} \sum_{x''y''} p(x''y'') f(x'y',x'',y'').
\]
The system of equations $\TT Q^* = Q^*$ is essentially as hard to solve (or even harder) than the Bellman optimality 
equations for $V^*$ stated earlier. However, one can develop an algorithmic approach toward finding an optimal 
transition coupling 
by 
drawing inspiration from the literature on regularized dynamic programming.

In particular, we develop below an analogue of an entropy-regularized policy iteration scheme that is known under many 
names in the RL literature: Natural Policy Gradients by \citet{K01} and \citet{AKLM21}, MDP-Expert by 
\citet{even-dar09OnlineMDP}, Mirror-Descent Policy Iteration by \citet{GSP19}, POLITEX by \citet{LABWBS19}, Policy 
Mirror Descent by \citet{AKLM21}, and the list goes on. This method can be directly adapted to our setting as follows. 
Starting from an arbitrary initial transition coupling $\pi_1$, SPI performs the 
following sequence of updates for each $k=1,2,\dots,K$:
\begin{itemize}
   \item Round the transition coupling $\pi_k$ to $\tpi_k = \rho(\pi_k)$,
   \item update the Q-function by solving the fixed-point equation $Q_{k} = \TT^{\tpi_k} Q_{k}$,
   \item if $k$ is odd, then update the transition coupling as
   \[
    \pi_{k+1}(x'y'|xy) = \frac{\pi_{k}(x'y'|xy) \exp(-\eta Q_{k}(xy,x'y'))}{\sum_{y''} \pi_{k}(x'y''|xy) \exp(-\eta 
Q_{k}(xy,x'y''))} \PPX(x'|x),
   \]
   \item else if $k$ is even, then update the transition coupling as
   \[
    \pi_{k+1}(x'y'|xy) = \frac{\pi_{k}(x'y'|xy) \exp(-\eta Q_{k}(xy,x'y'))}{\sum_{x''} \pi_{k}(x''y'|xy) \exp(-\eta 
Q_{k}(xy,x''y'))} \PPY(y'|y).
   \]
\end{itemize}
As in the case of SVI, it is easy to verify that these transition couplings satisfy the required marginal constraints. 
Furthermore, one can verify that the updates defined above exactly correspond to running an instance of Mirror Sinkhorn 
\citep{BB23} in each state-pair $xy$ with the sequence of cost functions 
$Q_1(xy,\cdot),Q_2(xy,\cdot),\dots,Q_K(xy,\cdot)$, similarly how the entropy-regularized policy iteration methods run 
an instance of entropic mirror descent in each state.
We discuss various aspects of this algorithm below.

\subsection{Practical implementation}
\begin{wrapfigure}{r}{0.5\textwidth}
\vspace{-.5cm}
\begin{minipage}{0.5\textwidth}
\begin{algorithm}[H]
\caption{Sinkhorn Policy Iteration}\label{alg:SPI}
\textbf{Input: } $\PPX$, $\PPY$, $c$, $\eta$, $\gamma$, $K$, $m$\\
\textbf{Initialise: }  $\pi_{1} \gets \PPX \otimes \PPY$\;
\For {$k=1,...,K-1$}{
	$\tpi_k \gets \rho(\pi_k)$\;
	$Q \gets \pa{\TT^{\tpi_k}}^{m} Q$\;   
	$\pi_{k+1} \gets \textbf{update}(\pi_{k},Q)$\;	\vspace{-1.1em}\Comment{Equation \ref{eq:update}} 
}
$\muout \gets \frac{1}{K}\sum _{k=1}^{K}\tmu_k$\;
$\piout \gets \pi_{\muout}$\;
$V^\piout \gets \textbf{evaluate}(\piout)$\;
\textbf{Output:} $\piout$, $V^\piout$ \Comment{Final coupling}
\end{algorithm}
\end{minipage}
\vspace{-.5cm}
\end{wrapfigure}
Like SVI, this method can be seen as performing online Mirror Sinkhorn updates in each state pair $xy$ 
with a sequence of cost functions $Q_k$, which are computed via solving the linear system of Bellman equations $Q_k = 
\TT^{{\tpi}_k} Q_k$. The occupancy couplings produced by SPI are denoted by $\tmu_k = \mu^{\tpi_k}$, and the final 
output of the method is produced by averaging these occupancies as $\muout = \frac 1K \sum_{k=1}^K\tmu_k$, and then 
extracting the transition coupling $\piout = \pi_{\muout}$. Notably, there is no need to round this transition coupling 
since 
each $\tmu_k$ satisfies all constraints by construction, and so does their average $\muout$ due to the linearity of the 
constraints.

One advantage of SPI over SVI is that finding an exact solution for the fixed-point equation $Q_{k} = \TTpik 
Q_{k}$  is easier than computing the fixed points required by SVI, thanks to the fact that this is a linear system of 
equations. A downside of the method is that it requires to run the rounding procedure after each update, at least for 
the theoretical guarantees to remain valid. The impact of these steps may however be negligible in practical 
implementations. Similarly to SVI, the ideal updates of SPI can be approximated by applying the Bellman operator to the 
Q-functions only a small number of times $m$, with $m=\infty$ corresponding to the ideal implementation analyzed 
above. We present a pseudocode for SPI as Algorithm~\ref{alg:SPI}.

\subsection{Convergence guarantees}
In what follows, we show the following performance guarantee for SPI.
\begin{theorem}\label{thm:main2}
Suppose that Sinkhorn Policy Iteration is run for $K$ steps with regularization parameter $ \eta = 
\frac {1-\gamma} {3\infnorm{c}} \sqrt{\frac{8\log |\X||\Y|}{K}}$, and 
initialized with the uniform coupling defined for each $xy,x'y'$ as $\pi_1(x'y'|xy) = \frac{1}{|\X||\Y|}$. Then, for 
any $x_0y_0\in\X\Y$, the output satisfies 
$V^\piout(x_0y_0) \le \mathbb{W}_\gamma(\MX,\MY;c,x_0,y_0) + \varepsilon$ if the number of iterations is at least
\[
 K \geq \frac{5\infnorm{c}^2 \log |\X||\Y|}{(1-\gamma)^4 \varepsilon^2}.
\]
\end{theorem}
As the analyses of all regularized policy iteration methods listed above, this one also starts with establishing the following claim 
that corresponds to the classic \emph{performance difference lemma}
(often attributed to \citealp{KL02}, but proposed much earlier in works like \citealp{Cao99} and even \citealp{How60}). 
To state the result, we define $V_k(xy) = \sum_{x'y'} \tpi_k(x'y'|xy) Q_k(xy,x'y')$ and the operators $E_+:\real^{\X\Y\times\X\Y} 
\ra \real^{\X\Y}$ and $E_-:\real^{\X\Y\times\X\Y} \ra \real^{\X\Y}$ with each element given by $(E_-V)(xy,x'y') = V(xy)$ and 
$(E_+V)(xy,x'y') = V(x'y')$. Then, the following bound holds on the 
instantaneous regret of SPI in round $k$.
\begin{lemma} $\iprod{\tmu_k - \mu^*}{c} = \iprod{\mu^*}{E_-V_k - Q_k}$.
\end{lemma}
\begin{proof}
 The proof follows from elementary properties of occupancy couplings. First, note that by the definition of the value 
function $V_k$, we have
 \[
  \iprod{\tmu_k}{c} = (1-\gamma) \iprod{\nu_0}{V_k}.
 \]
 Furthermore, by multiplying both sides of the Bellman equations $Q_k = c + \gamma E_+ V_k$ with $\mu^*$ and using 
the flow constraints $E_+\transpose \mu^* = \gamma E_-\transpose \mu^* + (1-\gamma) \nu_0$, we obtain
 \begin{align*}
  \iprod{\mu^*}{Q_k} &= \iprod{\mu^*}{c + \gamma E_+ V_k} = 
  \iprod{\mu^*}{c + E_-V_k} - (1-\gamma) \iprod{\nu_0}{V_k}.
 \end{align*}
 The result follows after reordering the terms.
\end{proof}
Given the above lemma, we can readily express the regret of SPI as follows:
\begin{align*}
 \sum_{k=1}^K \iprod{\tmu_k - \mu^*}{c} &= \sum_{k=1}^K \iprod{\mu^*}{E_-V_k - Q_k}
 \\
 &= \sum_{xy} \nu^*(xy) \sum_{k=1}^K \iprod{\pi^*(\cdot|xy) - \tpi_k(\cdot|xy)}{Q_k(xy,\cdot)},
\end{align*}
where we can recognize the regret of Mirror Sinkhorn in each state pair $xy\in\X\Y$. Thus, applying the bound of 
Theorem~3.1 of \citet{BB23}  to each of these terms (while noting that $\infnorm{Q_k} \le 
\infnorm{c}/\pa{1-\gamma}$ holds for all $k$) gives
\[
 \sum_{k=1}^K \iprod{\pi^*(\cdot|xy) - \tpi_k(\cdot|xy)}{Q_k(xy,\cdot)} \le 
\frac{\DDKL{\pi^*(\cdot|xy)}{\pi_1(\cdot|xy)}}{\eta} + \frac{9\eta \infnorm{c}^2 K}{8(1-\gamma)^2},
\]
and thus putting the bounds together we obtain
\begin{align*}
 \iprod{\muout - \mu^*}{c} &=  \frac 1K \sum_{k=1}^K \iprod{\tmu_k - \mu^*}{c} \le 
\frac{\HH{\mu^*}{\mu_1}}{\eta K } + \frac{9\eta \infnorm{c}^2}{8(1-\gamma)^2}
\end{align*}
Now, setting $\pi_1$ as the uniform coupling, we can further upper bound the conditional relative entropy as 
$\HH{\mu^*}{\mu_1} \le \log |\X||\Y|$, and after setting $\eta = \frac {1-\gamma} {3\infnorm{c}} \sqrt{\frac{8\log 
|\X||\Y|}{K}}$, the bound becomes
\[
 \iprod{\muout - \mu^*}{c} \le \frac{6\infnorm{c}}{1-\gamma}\sqrt{\frac{\log |\X||\Y|}{8K}}.
\]
Finally, note that 
\begin{align*}
V^\piout(x_0y_0) - \mathbb{W}_\gamma(\MX,\MY;c,x_0,y_0) &= \frac 1 {1-\gamma} \iprod{\muout - \mu^{*}}{c} \leq \frac{6\infnorm{c}}{(1-\gamma)^2}\sqrt{\frac{\log |\X||\Y|}{8K}}.
\end{align*}
Using a crude upper bound $36/8\leq 5$ verifies the claim of Theorem~\ref{thm:main2}.

\subsection{Relation to Sinkhorn Value Iteration}
While on the surface, SPI may seem only loosely related to SVI, a closer connection can be drawn by making the 
following observations. First, observe that the transition-coupling updates exactly match the updates of 
Sinkhorn Value Iteration, although there is an apparent difference in how the Q-functions are defined. To expose the 
similarity 
between the two methods better, let us consider an even round in which $\pi_k$ satisfies the $\X$-marginal conditions 
$\sum_{y'}\pi_k(x'y'|xy) = P(x'|x)$. Thus, introducing the notation $\VX(xy,x') = \sum_{y'} 
\frac{\pi_{k}(x'y'|xy)}{P(x'|x)} Q_k(xy,x'y')$, we can multiply the Bellman equations by $\pi_k(x'y'|xy) / 
P(x'|x)$ and sum them up to obtain
\begin{align*}
 \VX(xy,x') =& \sum_{y'}\frac{\pi_k(x'y'|xy)}{P(x'|x)}\pa{c(xy) + \gamma \sum_{x''y''} \pi_{k}(x''y''|x'y') 
Q_k(x'y',x''y'')}
 \\
 =& \sum_{y'}\frac{\pi_k(x'y'|xy)}{P(x'|x)}\pa{c(xy) + \gamma \sum_{x''} P(x''|x') \VX(x'y',x'')}
 \\
 \approx& -\frac{1}{\eta} \log \sum_{y'}\frac{\pi_k(x'y'|xy)}{P(x'|x)}\exp\pa{-\eta \pa{c(xy) + \gamma \sum_{x''} 
P(x''|x') \VX(x'y',x'')}},
\end{align*}
where the approximation is accurate as $\eta$ approaches zero. Thus, for small values of $\eta$, the Sinkhorn--Bellman 
operator used by Sinkhorn Value Iteration is an accurate approximation of the Bellman operator used by Sinkhorn Policy 
Iteration, and thus one may reasonably expect their respective fixed points to be close as well (this intuition may, 
however easily fail and is not necessary for our analysis above).\looseness=-1

\newpage
\section{Auxiliary proofs and technical results}

In this appendix we prove several technical results from the main text.

\subsection{Proof of Proposition \ref{prop:MDequivalence}}\label{sec:MDequivalence}
We prove the claim by showing that the solution of the constrained optimization problem of Equation~\eqref{eq:MDupdate} 
is equivalent to the transition-coupling update rule specified in Equation~\eqref{eq:update}. To this end, we study the 
Lagrangian of the optimization problem~\eqref{eq:MDupdate} corresponding to the update for odd rounds, $\BX$, and note that 
the update rule for even rounds,  $\BY$, can be worked out analogously. By introducing Lagrange multipliers $\VX(xy,x')$ 
for each constraint in $\BX$, we obtain the Lagrangian
\begin{align*}
 \LL(\mu;\VX) =& \iprod{\mu}{c} + \frac{1}{\eta} \HH{\mu}{\mu_k} 
 \\
 &+ \sum_{xy,x'} \VX(xy,x') 
 \pa{\pa{\gamma \sum_{x''y''} \mu(x''y'',xy)  + (1-\gamma) \nu_0(xy)} \PPX(x'|x) - \sum_{y'} \mu(xy,x'y') }
 \\
 =&
 \sum_{xy,x'y'} \mu(xy,x'y') \pa{c(xy) + \gamma \sum_{x''} \PPX(x''|x')\VX(x'y',x'') - \VX(xy,x')} 
 \\
 &+ (1-\gamma) \sum_{xy,x'}\nu_0(xy) \PPX(x'|x)\VX(xy,x') + \frac{1}{\eta} \HH{\mu}{\mu_k}.
\end{align*}

A quick calculation (cf.~Appendix~A.1 of \citealp{NJG17}) shows that the derivative of $\HH{\mu}{\mu_k}$ satisfies 
\[
\frac{\partial \HH{\mu}{\mu_k}}{\partial \mu(xy,x'y')} = \log \pi_\mu(x'y'|xy) - \log \pi_{\mu_k}(x'y'|xy) = \log 
\pi_\mu(x'y'|xy) - \log \pi_k(x'y'|xy),
\]
where we have used that $\pi_{\mu_k} = \pi_k$ holds by definition of $\mu_k$ and $\pi_k$. To proceed, for a fixed 
$\VX$, we set the gradient of the Lagrangian to zero and solve for the transition coupling $\pi_{k+1}$, which gives
\begin{equation*}
     \pi_{k+1}(x'y'|xy) = \pi_{k}(x'y'|xy) \exp\pa{-\eta \pa{ c(xy) + \gamma \sum_{x''} \PPX(x''|x') \VX(x'y',x'') - 
\VX(xy,x')}}.
\end{equation*}
Then, the correct choice of $\VX$ has to be such that the constraint $\sum_{y'} \pi_{k+1}(x'y'|xy) = \PPX(x'|x)$ is 
satisfied. To see this, suppose that this condition is indeed verified and that $\mu_{k+1}$ is the occupancy coupling 
associated with $\pi_{k+1}$. We need to show that $\mu_{k+1}$ indeed verifies the condition defining $\BX$. To this 
end, notice that $\mu_{k+1}$ satisfies Equation~\eqref{eq:const_flow}, and thus the condition can be simply written as
\[
 \sum_{y'} \mu_{k+1}(xy,x'y') = \pa{\sum_{x''y''} \mu_{k+1}(xy,x''y'')} \PPX(x'|x),
\]
which, after recalling the relation $\pi_{k+1}(x'y'|xy) = \frac{\mu(xy,x'y')}{\sum_{x''y''} \mu(xy,x''y'')}$ can be 
indeed seen to hold if $\sum_{y'} \pi_{k+1}(x'y'|xy) = \PPX(x'|x)$ is true.

Enforcing this constraint gives the following expression for $\VX(xy,x')$:
\begin{equation*}
 \VX(xy,x') = -\frac{1}{\eta} \log \sum_{y'} \frac{\pi_{k}(x'y'|xy)}{\PPX(x'|x)} \exp\pa{-\eta \pa{ c(xy) + \gamma 
\sum_{x''} \PPX(x''|x')  \VX(x'y',x'')}}.
\end{equation*}
To conclude, we need to ensure that this system of equations has a unique solution. In order to do this, 
we recall the definition of the Bellman--Sinkhorn operator $\TT^{\pi_k}_\X: \real^{\X\Y\times\X} \rightarrow \real^{\X\Y\times\X}$, 
acting on a function $f$ as
\begin{equation*}
 (\TT^{\pi_k}_\X f)(x'y',x'') = -\frac{1}{\eta} \log \sum_{y'} \frac{\pi_{k}(x'y'|xy)}{\PPX(x'|x)} \exp\pa{-\eta \pa{ 
c(xy) + \gamma \sum_{x''} \PPX(x''|x') 
f(x'y',x'')}}.
\end{equation*}
With this notation, we can directly verify that $\VX$ satisfies $\TT^{\pi_k}_\X \VX = \VX$. Furthermore, it can be 
shown that the operator $\TT^{\pi_k}_\X$ is a $\gamma$-contraction in supremum norm (cf.~Lemma 
\ref{lemma:contraction}), and thus it has a unique fixed point by the Banach fixed-point theorem. We finally note that 
defining $Q_k(xy,x'y')=\pa{ c(xy) + \gamma \sum_{x''} \PPX(x''|x') \VX(x'y',x'')}$, the transition-coupling update 
derived above can be rewritten as 
\begin{equation}
     \pi_{k+1}(x'y'|xy) = \frac{\pi_{k}(x'y'|xy) \exp\pa{-\eta Q_k(xy, x'y')}}{\sum _{y'} \pi_{k}(x'y''|xy) 
\exp\pa{-\eta Q_k(xy, x'y'')}}\PPX(x'|x).
\end{equation}
This concludes the proof.
\qed

\subsection{Rounding procedure and proof of Lemma \ref{lemma:Delta_to_delta}}\label{app:rounding}
\label{app:Delta_to_delta}

\begin{wrapfigure}{r}{0.5\textwidth}
\vspace{-.5cm}
\begin{minipage}{0.51\textwidth}
\begin{algorithm}[H]
\caption{Rounding procedure for couplings\hspace*{-.2cm}}\label{alg:round}
\textbf{Input: } approximate coupling $F$, margins $p$, $q$\\
$X \gets \diag(\min(p/(F\cdot\mathbf{1}),\mathbf{1}))$\;
$F' \gets XF$\;
$Y \gets \diag(\min(q/(F'^\top\cdot\mathbf{1}),\mathbf{1}))$\;
$F'' \gets F'Y$\;
$\text{err}_p=p-F''\cdot \mathbf{1}$, $\text{err}_q=q-F''^\top\cdot\mathbf{1}$\;
\textbf{Output: } $G\gets F'' + \text{err}_p\text{err}_q^\top / \onenorm{\text{err}_p}$
\end{algorithm}
\end{minipage}
\vspace{-.1cm}
\end{wrapfigure}

We adapt the rounding procedure stated as Algorithm~2 of \citet{altschuler2017near} and reproduced here as 
Algorithm~\ref{alg:round} (where $/$ denotes element-wise division in the pseudocode). Formally, for two probability 
distributions $p \in \Delta(\X)$  and $q \in \Delta(\Y) $, the set of valid couplings is $\mathcal{U}_{p,q} 
= \{ P \in \mathbb{R}_{+}^{\X\Y} : P \cdot \mathbf{1} = p ;\linebreak P^T \cdot \mathbf{1} = q \} $. 
 For a nonnegative matrix $ F \in 
\mathbb{R}_{+}^{\X \Y} $, 
the rounding procedure outputs 
a valid coupling $ \rho(F, p, q) \in \mathcal{U}_{p,q} $ which, by Lemma~7 of \citet{altschuler2017near}, 
satisfies
\begin{equation*}
	\norm{\rho(F, p, q)-F}_1 \leq 2 \pa{\norm{F\cdot \mathbf{1}-p}_1 + \norm{F^T \cdot \mathbf{1} - q}_1
}.
\end{equation*}

We will now define the rounding procedure for a (not necessarily valid) transition coupling $ \pi \in 
\mathbb{R}^{\X\Y\times\X\Y}$ by using the aforementioned rounding procedure at each state pair as
\begin{equation*}
	\tilde{\pi}(\cdot|xy) = \rho(\pi(\cdot|xy), \PPX(\cdot |x), \PPY(\cdot |y)).
\end{equation*}
With some abuse of notation, we will write the resulting transition coupling as $\wt{\pi} = \rho(\pi)$ and the 
associated occupancy coupling as $\wt{\mu} = \mu^{\wt{\pi}} = \rho(\mu)$.
Because of the correctness of the original rounding procedure, this transition coupling is valid, and so is the 
associated occupancy coupling. We can now proceed to the proof of Lemma~\ref{lemma:Delta_to_delta}.
\begin{proof}[Proof of Lemma~\ref{lemma:Delta_to_delta}]
Let $ \mu\in \mathbb{R}^{\X\Y\times\X\Y}_{+} $, and as before define $ \nu_{\mu}(xy) = \sum_{x'y'} \mu(xy, x'y') $ and
\begin{align*}
	\pi_\mu(x^{\prime}y^{\prime}|xy) =
\begin{cases}
	\frac{\mu(xy,x^{\prime}y^{\prime})}{\nu_{\mu}(xy)} \text{ if } \nu_\mu(xy) \neq 0, \\
	\PPX(x^{\prime}|x)\PPY(y^{\prime}|y) \text{ otherwise} .
\end{cases}
\end{align*}
For arbitrary state pairs $ xy $, we use Lemma 7 of \citet{altschuler2017near} to obtain that
\begin{equation*}
	\norm{\tilde{\pi}_\mu(\cdot |xy) - \pi_\mu(\cdot |xy)}_1 \leq 2 \left[
	\sum_{x^{\prime}} \left\lvert\PPX(x^{\prime}|x) - \sum_{y^{\prime}} \pi_\mu(x^{\prime}y^{\prime}|xy)\right\rvert + \sum_{y^{\prime}}
\left\lvert\PPY(y^{\prime}|y) - \sum_{x^{\prime}} \pi_\mu(x^{\prime}y^{\prime}|xy)\right\rvert \right]	.
\end{equation*}
Now, multiplying by $ \nu_\mu(xy) $ and summing over $ xy $, we get 
\begin{align*}
	&\phantom{=}\sum_{xy} \nu_\mu(xy)\norm{\tilde{\pi}_{\mu}(\cdot |xy) - \pi_{\mu}(\cdot |xy)}_{1}\\
	&\leq 2 \left[
	\sum_{xyx^{\prime}}\nu_\mu(xy)\left\lvert\PPX(x^{\prime}|x) - \sum_{y^{\prime}}
\pi_{\mu}(x^{\prime}y^{\prime}|xy)\right\rvert +
\sum_{xyy^{\prime}}\nu_\mu(xy)\left\lvert\PPY(y^{\prime}|y) - \sum_{x^{\prime}} 
\pi_{\mu}(x^{\prime}y^{\prime}|xy)\right\rvert \right] \\
									    &= 2 \left[
	\sum_{xyx^{\prime}}\left\lvert\nu_\mu(xy)\PPX(x^{\prime}|x) - \sum_{y^{\prime}}\mu(xy,x^{\prime}y^{\prime})\right\rvert +
\sum_{xyy^{\prime}}\left\lvert\nu_\mu(xy)\PPY(y^{\prime}|y) - \sum_{x^{\prime}} \mu(xy,x^{\prime}y^{\prime})\right\rvert \right] \\
									    &= 2 \delta_\X(\mu) + \delta_\Y(\mu) =  2 \delta(\mu),
\end{align*}
where we used the fact that $ \mu(xy, x^{\prime}y^{\prime}) = \nu_\mu(xy) \pi_{\mu}(x^{\prime}y^{\prime}|xy) $ for any $
xy, x^{\prime}y^{\prime} $, and the definitions $\delta_\X$, $\delta_\Y$ and $\delta$. Finally, we 
notice that $ \Delta(\mu) =\sum_{xy} \nu_\mu(xy)\norm{\tilde{\pi}_{\mu}(\cdot |xy) -
\pi_{\mu}(\cdot |xy)}_{1} $, which concludes the proof.
\end{proof}

\subsection{The contraction property of the Bellman--Sinkhorn operator}\label{sec:contraction}
\begin{lemma}
\label{lemma:contraction}
Let $\pi:\X\Y\ra\Delta_{\X\Y}$ be arbitrary and consider the associated Bellman--Sinkhorn operator $\mathcal{T}^\pi$ 
acting on a function $f:\X\Y\times\Y\ra\real$ as 
\[
(\mathcal{T}^\pi f)(x'y',x'')  = -\frac{1}{\eta} \log \sum_{y'} 
\frac{\pi(x'y'|xy)}{P(x'|x)}
\exp\pa{-\eta \pa{ c(xy) + \gamma \sum_{x''} P(x''|x')  f(x'y',x'')}}.
\]
Then, $\mathcal{T}^\pi$ is a $\gamma$-contraction for the supremum norm $\infnorm{\cdot}$ , that is, for any two 
functions $f_1,f_2:\X\Y\times\X$, we have  
\[
\infnorm{\mathcal{T}^\pi f_1 - \mathcal{T}^\pi (f_2) 
} \leq \gamma \infnorm{f_1 - f_2}.
\]
\end{lemma}
\begin{proof}
The claim easily follows from using the standard fact that the function $g_p(z) = \log \sum_{y'} p(y') e^z(y')$ is 
$1$-smooth with respect to the supremum norm, so that for any two vectors $z$, we have $\abs{g_p(z) - g_p(z)} \le 
\infnorm{z - z'}$. To apply this result, we define $q(y'|xy) = 
\pi(x'y'|xy) / P(x'|x)$ and $z_1(xy,x'y') = c(xy) + \gamma \sum_{x''} P(x''|x') f_1(x'y',x'')$ and 
$z_2(xy,x'y') = c(xy) + \gamma \sum_{x''} P(x''|x') f_2(x'y',x'')$, so that we can write
\begin{align*}
 \infnorm{\mathcal{T}^\pi f_1 - \mathcal{T}^\pi f_2 } &= \max_{xy,x'} \abs{g_{q(\cdot|xy)}(z_1(xy,x'\cdot)) - 
g_{q(\cdot|xy)}(z_2(xy,x'\cdot))}
\\
&\le \max_{xy,x'} \infnorm{z_1(xy,x'\cdot) - z_2(xy,x'\cdot)} = \infnorm{z_1 - z_2} \le \gamma \infnorm{f_1 - f_2},
\end{align*}
where the last step follows from the straightforward calculation
\begin{align*}
 \infnorm{z_1 - z_2} = \sup_{xy,x'y'} \sum_{x''} P(x''|x') \abs{f_1(x'y',x'') - f_2(x'y',x'')} \le \infnorm{f_1 - f_2}.
\end{align*}
This concludes the proof.

\end{proof}

\newpage
\section{Additional experimental results}\label{app:experiments}

In this appendix we present the results of additional experiments not included in the main text.

\subsection{Impact of the regularization parameter $\eta$}

Besides the parameter $m$ that we have already studied experimentally in Section~\ref{sec:exps}, the only tuning 
parameter of SVI and SPI is the regularization parameter $\eta$, which takes the role of a learning rate. 
In this experiment, we study two random walks run on two separate 4-room environments \citep{sutton1999}, with two 
separate reward functions $r_{\X}$ and $r_{\Y}$ that together define the ground cost function $c(x,y) = 
\abs{r_{\X}(x) - r_{\Y}(y)}$ for each state pair. The results of this study for $K=2\cdot 10^{4}$ iterations are shown 
in Figure~\ref{fig:convergence}. 
The error is computed as the difference between the distance estimate produced by the algorithm and a near-optimal 
distance obtained by running Algorithm \ref{alg:main} for a very small value of $\eta$ and a large number of 
iterations. We observe that higher values of $\eta$ lead to faster error reduction in the initial steps, but 
eventually prevent convergence to the true solution. In contrast, choosing smaller learning rates enables convergence 
to better 
solutions, at the cost of making slower progress initially. An intuitive explanation for this is that for larger values 
of $\eta$, the iterates converge rapidly to the broad proximity of an optimal solution, but then reach a cycle 
where they continue to perform large updates to the transition coupling which results in large constraint violations, 
which necessitate large updates, and so on. This is formally verified by the fact that our bound on the constraint 
violation terms in Theorem~\ref{thm:cumulative_delta} are increasing for large values of $\eta$. 
Notably, employing a time-dependent learning-rate schedule (inspired by \citet{BB23}) with $\eta_k \sim 1/\sqrt{k}$ 
leads to the best performance. This strategy leverages faster convergence initially and, for sufficiently large 
number of iterations, also achieves near-optimal solutions.

\begin{figure}[h]
     \centering
     \begin{subfigure}[b]{0.45\textwidth}
         \centering
         \includegraphics[width=\textwidth]{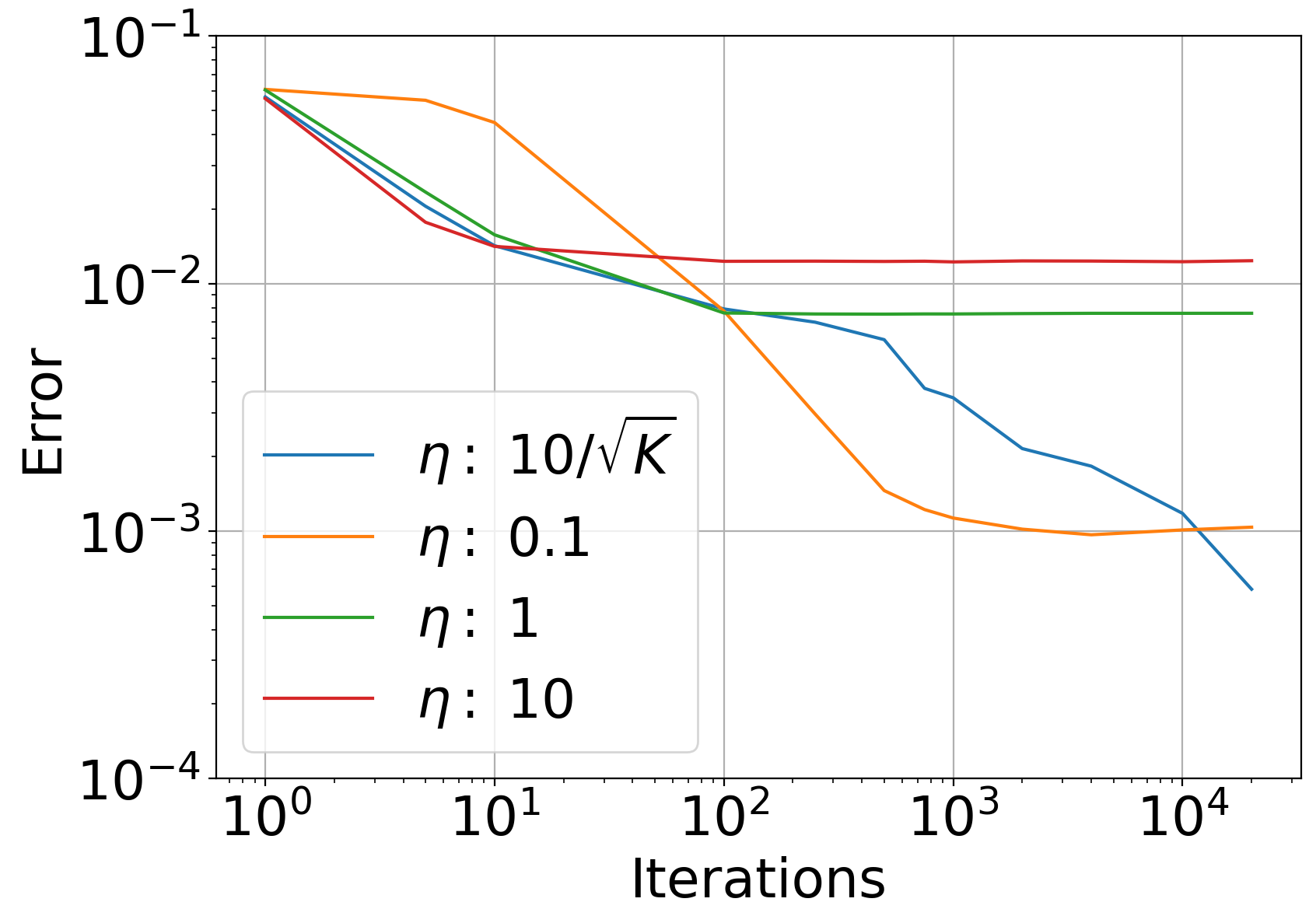}
         \caption{Algorithm \ref{alg:main}}
         \label{fig:convergence_alg1}
     \end{subfigure}
     \hfill
     \begin{subfigure}[b]{0.45\textwidth}
         \centering
         \includegraphics[width=\textwidth]{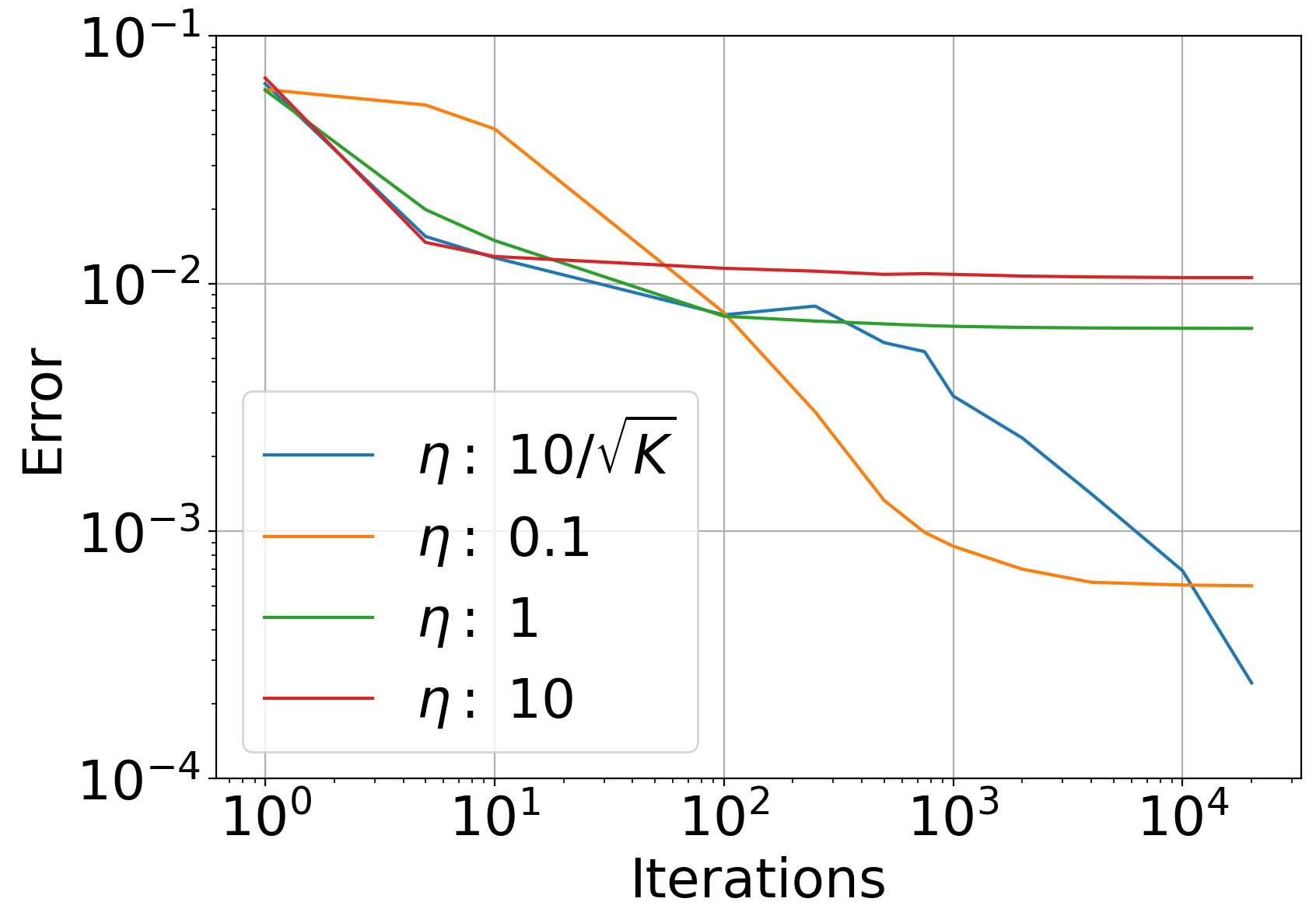}
         \caption{Algorithm \ref{alg:SPI}}
         \label{fig:convergence_alg2}
     \end{subfigure}
        \caption{Error of estimated transport cost as a function $k$, for various choices of $\eta$.}
        \label{fig:convergence}
\end{figure}

\subsection{Comparison with alternative methods}
We now turn to studying the computational complexity of our algorithms, and compare them empirically to some existing 
methods that have been proposed for computing optimal transport distances and bisimulation metrics between Markov 
chains. Specifically, we will focus here on two previous methods: the method of \citep{OMN21} that we refer to as 
``EntripicOTC'' and Algorithm~2 of \citep{BWW24} that we call ``dWL''.  We adapt both of these methods with some minor 
changes to our setting. First, we remove one scaling factor from the transport cost in the definition of the distance 
defined by \citet{BWW24} so that it matches ours. Second, the algorithm of \citet{OMN21} is originally defined 
for the infinite-horizon average-cost case, and thus we made appropriate changes to adapt it to the discounted case 
by replacing their approximate policy evaluation step by $T$ applications of the discounted Bellman evaluation 
operator. As pointed out in Appendix~\ref{sec:OT_discussion}, the resulting methods are closely related, and 
can be regarded as approximate dynamic programming methods for solving the MDP formulation of our optimal transport 
problem presented in Appendix~\ref{app:mdp}. We also recall that the algorithm proposed by \cite{kemertas2022} for the 
purpose of computing bisimulation metrics also falls into the same class of approximate dynamic programming methods, 
and nearly matches the method of \citep{BWW24}. The comparison below is based on the original Python 
implementation\footnote{\url{https://github.com/yusulab/ot_markov_distances}} of \citet{BWW24} and our own Python 
adaptation of the MATLAB code\footnote{\url{https://github.com/oconnor-kevin/OTC}} of \citet{OMN21}.

One difficulty that we had to face in these experiments is having to tune various hyperparameters of each method 
(such as number of iterations and regularization parameter), which can each influence the quality of the solution and 
the computation time. For a fair comparison between the methods, we have adopted the following procedure to obtain our 
results. First, we estimate a ground truth obtained by running one of the algorithms for a very high number of 
iterations and a very low level of regularization, and then use this ground truth as a comparator to adjust the 
hyperparameters of the algorithms so that they are close to this value in as little wall-clock time as possible.
While all the algorithms perform similar operations, their total runtimes turn out to be rather different and heavily 
dependent on problem parameters such as the discount factor $\gamma$.  The comparison between the resulting runtimes of 
each method is shown on Figure~\ref{fig:comp-time-comparison} as a function of the size of the Markov chains, and for a 
three different choices of the discount factor, for a set of randomly generated MDPs (following the setup described in 
Section~7.1 of \citealp{OMN21}).

First, we observe that EntropicOTC is a policy-iteration-like method, and as such it needs fewer iterations than the 
rest of the algorithms we tested to converge to the true cost, but each iteration requires running an expensive policy 
evaluation subroutine until convergence. This computational cost eventually adds up in a way that this algorithm has 
always ended up being the slowest among all that we have tested, although its performance has proved notably robust to 
changes in the discount factor $\gamma$. Second, we note that the updates of dWL are much cheaper to compute, especially 
for large regularization parameters. However, the errors of this value-iteration-like method compound much more rapidly 
than in the case of EntropicOTC, which makes it especially hard to tune the hyperparameters of this method. This problem 
is especially pronounced when the discount factor is large, which is the most interesting regime as it leads to 
distances with much stronger discrimination power. \looseness=-1

\begin{figure}[t]
	\center 
	\includegraphics[width=1\textwidth]{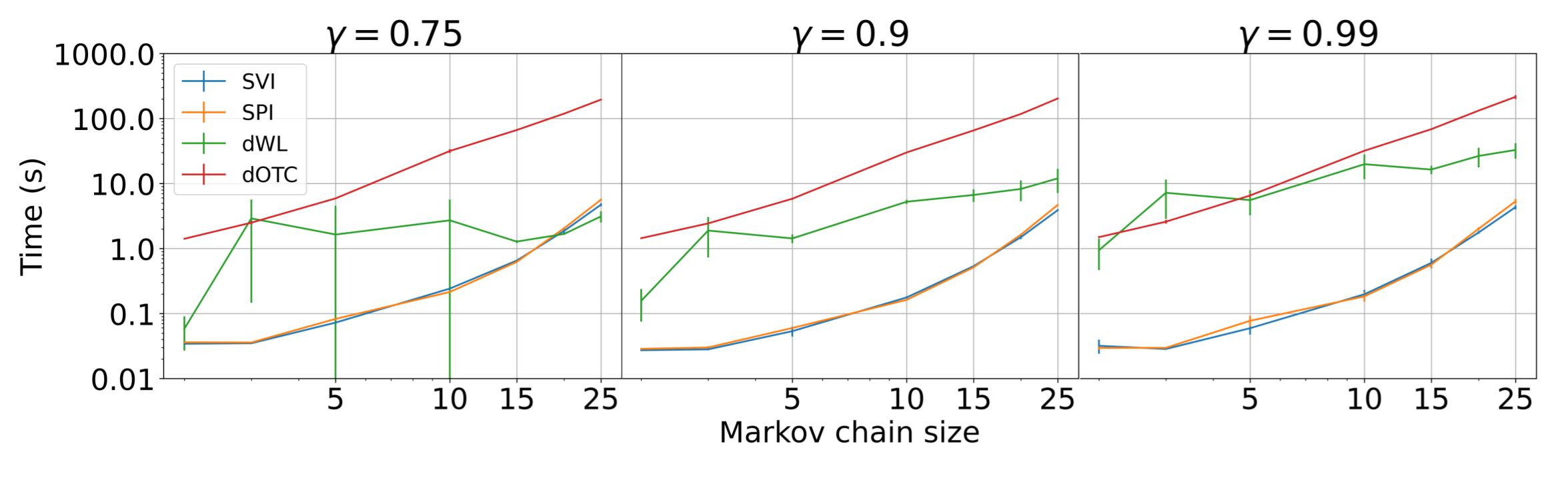} 
		\caption{Comparison of the computational time of the different methods proposed to obtain a near-optimal 
solution for different values of $\gamma$. For each Markov chain size, the results obtained in 5 randomly generated 
instances are compared, showing the standard deviation in the plot. Data is displayed on a log-log scale.}
     \label{fig:comp-time-comparison}
\end{figure}

The plots shown on Figure~\ref{fig:comp-time-comparison} indicate that our methods find optimal couplings consistently 
more efficiently than the other methods in the regime we studied, leading to up to 10 times faster runtimes. For the 
case of sufficiently small values of $\gamma$, the algorithm of \citet{BWW24} sometimes performs competitively, but the 
hyperparameters leading to good performance are much harder to find than in the case of our methods. In our experience, 
the massive speedup achieved by our methods can be largely ascribed to maintaining the transition couplings $\pi_k$ 
between iterations, as opposed to computing these afresh by running Sinkhorn's algorithm from scratch for each update 
as done by all other competing methods. Adjusting these other algorithms by maintaining the couplings in 
memory and using them to warm-start the subsequent updates makes them competitive with our methods, and in fact doing 
so makes them quite similar to SVI and SPI. Our algorithms use such warm-starts as a primary design choice 
as opposed to an obscure implementation detail, which is ultimately responsible for the computational efficiency of all 
these dynamic-programming methods.

The computational time per iteration of each of these methods grows roughly at a rate of $n^{4}$, with $n$ being the 
number of states of the Markov chains. In the case of our methods, this is easily explained by noting that applying the 
Bellman--Sinkhorn operators and updating the transition couplings requires $|\X|^2|\Y|^2$ operations in total. This 
matches the runtime necessary for running the regularized policy improvement subroutines employed by 
\citet{OMN21} and \citet{BWW24}, which consists of running an instance of Sinkhorn's algorithm in each pair of 
states. The computational cost of all these methods can be improved by leveraging the sparsity of transition kernels: 
in particular, if at most $S$ states are reachable with positive probability in both of the chains, the complexity of 
the updates can be trivially improved to $|\X||\Y|S^2$. We did not pursue this direction in our experiments as our goal 
was to compare the basic versions of each studied method, and we believe that our conclusions would not be altered if 
we were to implement this improvement for all methods.

\subsection{Optimal transport distances as similarity metrics}
We finally provide a range of experiments that illustrate how the optimal-transport distances we studied in this paper 
can be used to capture relationships and symmetries in groups of varying Markov chains. To this end, we have generated 
35 different ``4-room'' instances and computed their pairwise distances. 
Each instance differs in its initial state and position of the obstacles (amounting to changes in the transition 
kernel), while maintaining a fixed reward function, with one reward located in each room except for the upper left 
room. Crossing a door between each room results in a negative reward. For each instance, we have studied the Markov 
chain induced by the corresponding optimal policy (which amounts to taking the shortest path toward the closest 
positive reward, modulo the additional randomness inherent to the transitions).
Figure~\ref{fig:MDS} presents a 3D visualization (where the $z$-dimension is represented by a color gradient) generated 
using Multidimensional Scaling (MDS) \citep{kruskal1964} based on the computed pairwise distances.
One can observe a clear clustered structure, where instances with similar behaviors are grouped closely together. As 
the figure highlights, the resulting metrics capture the intuitive similarities and symmetries between each process, 
which indicates the potential usefulness of optimal transport distances and bisimulation metrics for comparing Markov 
chains under minimal structural assumptions made on the state spaces and the transition functions.

\begin{figure}[h!]
	\center 
	\includegraphics[width=0.6\textwidth]{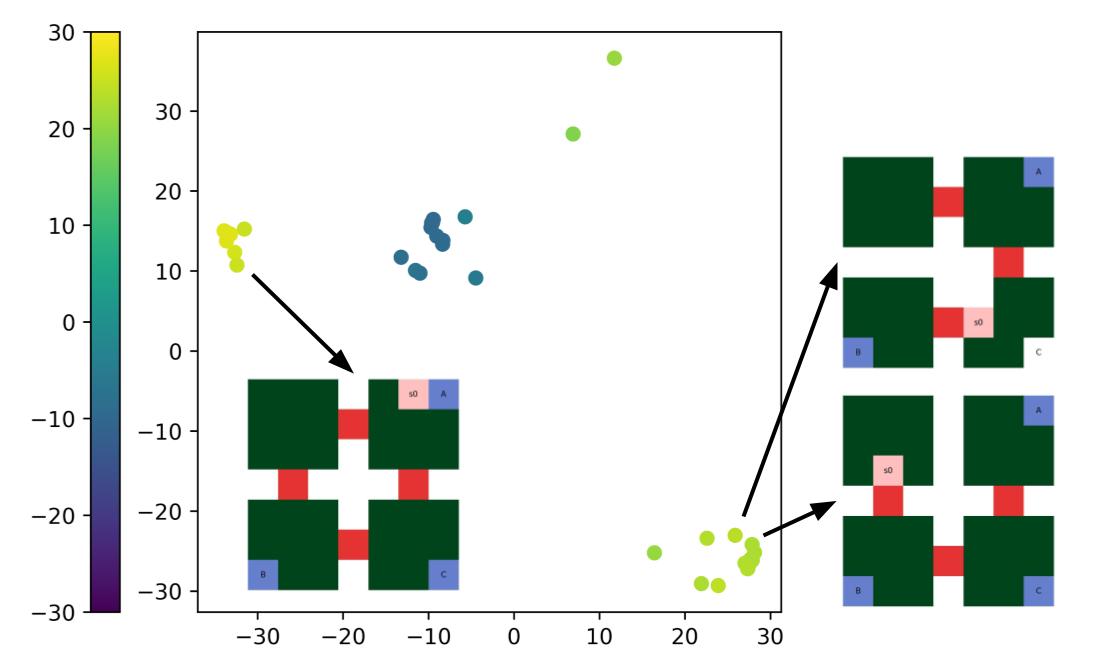} 
	\caption{Result of applying MDS to the pairwise distances between the set of 4-room instances studied. On the 
    plot, the first two coordinates of the MDS embedding are used as the spatial coordinates, and the third coordinate 
    is encoded via the color bar provided on the left-hand side of the axes.
	It can be observed how the elements in the same cluster present common features that differentiate them from those in another cluster. 
	In the examples shown in the figure we can see how the instances in which the closest reward involves crossing a door are concentrated in one cluster, 
	while the instances in which the reward and the initial state are located in the same room belong to a different 
cluster. The remaining clusters correspond to having to cross two doors for a reward (set of green points on the top), 
or having no reward that is accessible from the initial state (set of blue points in the middle, with large negative 
$z$-coordinates).}
	\label{fig:MDS}
\end{figure}

\end{document}